\documentclass[11pt]{article}
\usepackage[margin=0.75in]{geometry}
\usepackage{amsthm,amssymb,amsfonts}
\usepackage{amsopn,amstext}
\usepackage{bm}
\usepackage[graphicx]{realboxes}
\usepackage{float}
\usepackage{subcaption}
\usepackage{mathrsfs}
\usepackage{placeins}
\usepackage{chemformula}
\usepackage{listing}
\usepackage{xcolor}
\usepackage{multirow}
\usepackage{mathtools, eucal}
\usepackage{algorithm}
\usepackage[noend]{algpseudocode}
\usepackage[justification=centering]{caption}
\usepackage[bottom]{footmisc}
\usepackage{nccmath}
\usepackage{arydshln}
\usepackage{comment}
\usepackage{pdflscape}
\usepackage{slashbox}
\usepackage{color}

\usepackage{slashbox}

\newcommand{\argmin}{\mbox{argmin}}
\newcommand{\Bf}{\textbf}

\newcommand{\lp}{\left(}
\newcommand{\rp}{\right)}
\newcommand{\R}{\mathbb{R}}

\newcommand{\datayi}{y^{(i)}}

\newcommand{\dataxi}{\bm x^{(i)}}
\newcommand{\datasi}{\bm s^{(i)}}
\newcommand{\datari}{\bm r^{(i)}}

\newcommand{\dataXi}{\Bf X^{(i)}}
\newcommand{\dataMi}{\Bf M^{(i)}}
\newcommand{\dataHi}{\Bf H^{(i)}}
\newcommand{\ve}{\mbox{vec}}
\newcommand{\hv}{\mbox{hvec}}
\newcommand{\nct}{\frac{n(n+1)}{2}}
\newcommand{\RNum}[1]{\uppercase\expandafter{\romannumeral #1\relax}}

\theoremstyle{definition}
\newtheorem{definition}{Definition}[section]
\usepackage{color}
\newtheorem*{remark}{Remark}
\newtheorem{theorem}{Theorem}[section]
\newtheorem{corollary}{Corollary}[theorem]
\newtheorem{lemma}[theorem]{Lemma}

\usepackage[dvips]{epsfig}
\usepackage{flexisym}
\usepackage{verbatim}
\usepackage{listings}

\providecommand{\keywords}[1]{\textbf{\textit{Keywords.}} #1}

\title{\bf{Quadratic Surface Support Vector Machine with L1 Norm Regularization}}
\date{}

\usepackage[utf8]{inputenc}
\usepackage[T1]{fontenc}
\usepackage{authblk}
\usepackage{lipsum}
\newcommand\blfootnote[1]{%
  \begingroup
  \renewcommand\thefootnote{}\footnote{#1}%
  \addtocounter{footnote}{-1}%
  \endgroup
}

\author[1,*]{Ahmad Mousavi}
\author[2]{Zheming Gao}
\author[3]{Lanshan Han}
\author[3]{Alvin Lim}

\affil[1]{\footnotesize Institute for Mathematics and its Applications, University of Minnesota, Minneapolis, MN 55455, USA}
\affil[2]{\footnotesize College of Information Science and Engineering, Northeastern University, Shenyang, Liaoning 110819, China
}
\affil[3]{\footnotesize Precima, a NielsenIQ Company, Chicago, IL 60606, USA
}

\begin{document}
\maketitle

\abstract{We propose $\ell_1$ norm regularized quadratic surface support vector machine models for binary classification in supervised learning. We establish some desired theoretical properties, including the existence and uniqueness of the optimal solution, reduction to the standard SVMs over (almost) linearly separable data sets, and detection of true sparsity pattern over (almost) quadratically separable data sets if the penalty parameter on the $\ell_1$ norm is large enough. We also demonstrate their promising practical efficiency by conducting various numerical experiments on both synthetic and publicly available benchmark data sets.}

\keywords{binary classification, support vector machines, quadratic support vector machines, L1 norm regularization}

\blfootnote{This research was done in the Research and Development Department of Precima, which fully supported the work. The first two authors contributed equally.
\\
 Emails: amousavi@umn.edu, tonygaobasketball@hotmail.com, lanshan.han@nielseniq.com and alvin.lim@nielseniq.com.}


\section{Introduction}
Machine learning has recently found an extensive range of applications in various fields of contemporary science, including computer science, statistics, engineering, biology, and applied mathematics \cite{di2007survey,langley1995applications}. Several well-received industrial applications in which machine learning has performed well are healthcare, finance, retail, travel and energy \cite{hao2007new, monostori1996machine, luo2018benchmarking}. Nonetheless, compared to this significant applicability demonstrations in machine learning, rigorous theoretical studies in analyzing its models and verifying the correctness of obtained results can be improved.

Supervised learning is a major category of machine learning where labels are also available. Data classification is a vital task in supervised learning that extracts valuable information from available data, and exploits them to assign a new data point to a class \cite{ho2002complexity}. Proposed by Vapnik et al. \cite{cortes1995support} and well developed in the recent twenty years, support vector machine (\ref{formulation: SVM-hard margin}) is an effective and efficient tool for classification.  Given a labeled data set with two classes, the soft margin SVM model \cite{cortes1995support} finds a separating hyperplane that maximizes the minimum margin while minimizing the mis-classification penalty. It separates any (almost) linearly separable data set (almost) perfectly but fails to perform well when the data set is only nonlinearly separable. 

To deal with data sets that are only nonlinearly separable, SVM models with kernel functions were proposed \cite{cortes1995support}. The idea is to use a nonlinear kernel function and map the data to (a feature space embedded into) a higher dimensional space. Then,  the \ref{formulation: SVM-hard margin} model is applied for the classification of the mapped data in this feature space.
However, for a given data set, there is no general principle for the selection of an appropriate kernel function. Besides, the performance of kernel-based SVM models heavily depends on the kernel parameters \cite{cristianini2000introduction,scholkopf2001learning} and tuning those parameters often takes much effort. Therefore, training kernel-based SVM models often requires much extra computational efforts.

To take advantage of the idea of SVM while avoiding the challenges in using the nonlinear kernel trick, a kernel-free quadratic surface SVM (\ref{formulation:QSSVM}) model was proposed by Dagher \cite{dagher2008quadratic}.
Luo et al. recently developed its extension, the so called  soft margin quadratic surface SVM (\ref{formulation:SQSSVM}), that incorporates noise and outliers \cite{luo2016soft}. Both models seek a quadratic separating surface that maximizes an approximation of a relative geometric margin  \cite{dagher2008quadratic, luo2016soft}, while the SQSSVM handles non-separable data sets by penalizing mis-classifications of outliers and noise. These models are more tractable than kernel-based SVM models because the quadratic structure of the separation surface is explicit and clear \cite{luo2016soft}. There are other quadratic separating surface SVM models proposed in literature. For example, Luo et al. \cite{luo2016fuzzy} proposed a fuzzy kernel-free quadratic surface SVM model, and Bai et al. \cite{Bai2015} introduced a quadratic kernel-free least squares SVM and applied it to target diseases classification.

However, both QSSVM and SQSSVM fail to reduce to a hyperplane when a given data set is linearly separable; a natural expectation from an extension of the SVM. In addition, it is favorable to capture the possible sparsity of the hessian matrix of the separating quadratic surface especially when the number of features increases. A potential remedy to resolve these issues is adding an $\ell_p$ norm regularization term to the objective function. There is a broad literature available on the effects of this term on the optimal solution set for different values of $p$ in terms of uniqueness, stability, performance, and sparsity \cite{mousavi2019survey,saab2008stable, shen2018least,zhang2019robust,shen2019exact}. In particular, $\ell_1$ norm regularization  has demonstrated its effectiveness in finance, machine learning and sparse optimization \cite{dai2018generalized, lounici2009taking,mousavi2019solution,qiu2012fast}.

To eliminate the mentioned shortcomings of QSSVM and SQSSVM models and inspired by desirable properties of $\ell_1$ norm regularization, we propose their $\ell_1$ norm regularized extensions, namely,  \ref{formulation:QSSVM-vecl1} and \ref{formulation:SQSSVM-vecL1}; see Section \ref{sec:Quadratic Surface Support Vector Machines}. We rigorously study several interesting theoretical properties of these models, including solution existence and uniqueness, and reduction to the original SVM for (almost) linear separable data sets, and accurate sparsity pattern detection for (almost) quadratically separable data sets if the penalty parameter for the $\ell_1$ norm regularization is large enough.
To further examine our models and evaluate their numerical performance, we conduct computational experiments on both artificially generated and widely used benchmark data sets. The numerical results confirm that the proposed models outperform their parental models like the original \ref{formulation: SVM-hard margin}, \ref{formulation: SVM-soft margin}, \ref{formulation:QSSVM}, \ref{formulation:SQSSVM}, and also quadratic kernel-based model with respect to classification accuracy.

Roughly speaking, the penalty parameter for the $\ell_1$ norm regularization not only provides various interesting properties but also controls the curvature of the separating surface from quadratic to linear a priori; such that the larger this parameter is, the more this surface resembles a line. This key property opens up a range of separating surfaces for training a data set, which is particularly beneficial when a prior knowledge is available. In other words, if we  speculate a data set can be separated by a linear-type surface, we can select this parameter to be relatively large, while smaller values are appropriate choices for more quadratic-type surfaces.  On the contrary, the \ref{formulation: SVM-hard margin}, and \ref{formulation: SVM-soft margin} only capture hyperplanes and  \ref{formulation:QSSVM}, \ref{formulation:SQSSVM} only produce quadratic surfaces even if data set is linearly separable. 

The proposed L1-SQSSVM model is capable of handling data with many features since the $\ell_1$ norm regularization can find the important features, or the interactions between the features to improve the classification accuracy. It has many potential real-world applications in the field of industrial and management optimization, such as sentiment classification of customer reviews \cite{ghaddar2018high}, credit scoring \cite{gao2020kernel}, disease diagnosis \cite{Bai2015}, etc. 

In this paper, we propose two kernel-free SVM models with $\ell_1$ norm regularizations. The proposed models are both theoretically and numerically investigated. The main contributions of this paper include:
    
    \begin{enumerate}
        \item [(1)] To the best of our knowledge, this is the first study of proposing $\ell_1$ norm regularized kernel-free quadratic surface SVM models for binary classification. The proposed models generalize the classification capabilities of the standard linear SVM and the kernel-free quadratic surface SVM models on different types of data sets.
        
        \item [(2)] Some theoretical properties such as the solution existence, uniqueness and vanishing margin property are rigorously studied. Numerical experiments are conducted with some artificial data sets to validate the theoretical properties of the proposed models. Additional numerical experiments conducted with publicly available benchmark data sets verify the effectiveness and the efficiency of the proposed models, especially its potential of classifying data sets with many features.
    \end{enumerate}

The rest of this paper is organized as follows. In Section~\ref{sec:preliminaries}, we bring some preliminaries that facilitate writing and reading the paper. Section~\ref{sec:Quadratic Surface Support Vector Machines} lays down the foundation of different models discussed and proposed in this paper. We investigate different properties of new models \ref{formulation:QSSVM-vecl1} and \ref{formulation:SQSSVM-vecL1} in Section~\ref{sec:Theoretical Properties of L1 Norm Regularized QSSVMs}. We conclude the paper with various numerical experiments that demonstrate the behavior and performance of the proposed models in Section~\ref{sec:Numerical Experiments}.


\section{Notations and Preliminaries}\label{sec:preliminaries}
In this section, we introduce some notations as well as preliminaries to be used in this paper. This section is separated into two subsections. The first subsection mainly focuses on notations and preliminaries regarding matrices, while the second focuses on convex programs.
\subsection{Symmetric Matrices and Vectorization}
Throughout this paper, we use lower case letters to represent scalars, lower case bold letters to represent vectors, and upper case  bold letters to represent matrices. The set of real numbers is denoted by $\R$, the set of $n$-dimensional real vectors is denoted by $\R^n$, and the $n$-dimensional nonnegative orthant is denoted by $\R_+^n$. We use $\mathbf 1_n$ to represent the all one vector of length $n$, $\mathbf 0_{m\times n}$ to represent all 0 matrix of size $m \times n$, and $\mathbf I_n$ to represent the $n\times n$ identity matrix. Let $\mathbb S_n$ be the set of all real symmetric $n\times n$ matrices. For $\Bf A\in \mathbb S_n$, we write $\Bf A \succ 0$ to denote that $\Bf A$ is positive definite and $\Bf A \succeq 0$ to denote that $\Bf A$ is positive semidefinite. For a square matrix $\Bf A = [a_{ij}]_{i=1,\dots,n;j=1,\dots,n}\in \mathbb S_n$, its vectorization is the $n^2$-vector formed by stacking up the columns of $\Bf A$, i.e., the vectorization of $\Bf A$ is given by
$$\ve(\Bf A) = \left[a_{11},\dots, a_{n1},a_{12},\dots,a_{n2},\dots,a_{1n},\dots,a_{nn}\right]^T \in \R^{n^2}.$$
For symmetric matrix $\Bf A$, $\ve(\Bf A)$ contains repeated information, since all the information is contained in the $\frac{1}{2}n(n + 1)$ entries on and below the main diagonal. Therefore, we often consider half-vectorization of a symmetric matrix $\Bf A$, which is given by:
$$\hv(\Bf A) = \left[a_{11},\dots, a_{n1},a_{22},\dots,a_{n2},\dots,a_{nn}\right]^T \in \R^{\nct}.$$
It has been shown in \cite{magnus1980elimination} that, given $n$, there is a unique elimination matrix $\Bf L_n \in \R^{\nct \times n^2}$, such that $$\Bf L_n \ve(\Bf A) = \hv(\Bf A),\quad \forall \, \Bf A \in \mathbb S_n.$$ For example, when $n=3$, the elimination matrix
$$\Bf L_3 \, = \, \left[
\begin{array}{c;{2pt/2pt}c;{2pt/2pt}c}
  \begin{array}{ccc}1&0&0\\0&1&0 \\ 0&0&1\end{array} & \mathbf{0}_{3\times 3} & \mathbf{0}_{3 \times 3} \\\hdashline[2pt/2pt]
  \mathbf{0}_{2 \times 3} & \begin{array}{ccc} 0 & 1& 0\\ 0 & 0 & 1\end{array} & \mathbf{0}_{2\times 3} \\\hdashline[2pt/2pt]
  \mathbf{0}_{1 \times 3} & \mathbf{0}_{1\times 3} & \begin{array}{ccc} 0 & 0 & 1\end{array}
\end{array}
\right].$$
It is known that the elimination matrix $\Bf L_n$ has full row rank \cite{magnus1980elimination}, which is $\frac{1}{2}n(n+1)$. Conversely, for any $n$, there also exists a unique duplication matrix $\Bf D_n \in \R^{n^2 \times \nct}$ such that $$\Bf D_n \hv(\Bf A) = \ve(\Bf A), \quad \forall\, \Bf A \in \mathbb S_n.$$ When $n=3$, the duplication matrix
$$\Bf D_3 \, = \, \left[
\begin{array}{c;{2pt/2pt}c;{2pt/2pt}c}
  \begin{array}{ccc}1&0&0\\0&1&0 \\ 0&0&1\end{array} & \mathbf{0}_{3\times 2} & \mathbf{0}_{3 \times 1} \\\hdashline[2pt/2pt]
  \begin{array}{ccc} 0 & 1& 0\\ 0 & 0 & 0 \\0 &0 &0\end{array} & \begin{array}{cc} 0 & 0 \\1 & 0\\ 0 & 1\end{array}  & \mathbf{0}_{3\times 1} \\\hdashline[2pt/2pt]
  \begin{array}{ccc} 0 & 0& 1\\ 0 & 0 & 0 \\0 &0 &0\end{array} & \begin{array}{cc} 0 & 0 \\0 & 1\\ 0 & 0\end{array} & \begin{array}{c} 0 \\ 0 \\ 1\end{array}
\end{array}
\right].$$
It is also known that $\Bf D_n$ has full column rank, which is $\frac{1}{2}n(n+1)$. Moreover, $$\Bf L_n \Bf  D_n \, = \, \mathbf{I}_{\nct},$$
the $\nct \times \nct$ identity matrix. Given two matrices $\Bf A \in \R^{m\times n}$ and $\Bf B \in \R^{p \times q}$, the Kronecker product of them is written as
$$\Bf A \otimes \Bf B \, = \, \left[
\begin{array}{ccc}
a_{11} \Bf B & \cdots & a_{1n} \Bf B \\
\vdots & \vdots & \vdots \\
a_{m1} \Bf B & \cdots & a_{mn}\Bf B
\end{array}\right] \, \in \, \R^{mp \times nq}.$$
Let a symmetric matrix $\Bf M$ be partitioned as follows
\begin{equation}\label{eq:partitioned_matrix}
\Bf M \, = \, \left[
\begin{array}{cc}
\Bf A & \Bf B \\
\Bf B^T & \Bf C
\end{array}
\right].
\end{equation}
We state the following standard lemma regarding the positive definiteness of partitioned matrix $\Bf M$.

\begin{lemma}\label{lm:pd_partitioned}\cite[Proposition 16.1]{gallier2011schur}
Let matrix $\Bf M$ be partitioned as (\ref{eq:partitioned_matrix}). Then, $\Bf M$ is positive definite if and only if $\Bf C$ is positive definite and $\Bf A-\Bf B\Bf C^{-1}\Bf B^T$ is positive definite.
\end{lemma}
\subsection{Convex Optimization and Some Standard Results}
The SVM-type problems are typically modeled as convex optimization problems. We review a few related results from optimization theory in this subsection. We first consider the following quadratic program (QP):
\begin{equation}\label{eq:qp_func}
\begin{array}{ll}
  \min & \frac{1}{2} \bm x^T \Bf Q \bm x + \bm b^T \bm x\\[5pt]
  \mbox{s.t.} & \Bf  A\bm x \, \geqslant \,\bm c
\end{array}
\end{equation}
where $\Bf Q\in \mathbb S_p$ is a given symmetric matrix, $\bm b\in \R^p$ and $\bm c\in \R^q$ are given vectors, and $\Bf A\in \R^{q\times p}$ is a given matrix. We first provide a result regarding solution existence of the QP (\ref{eq:qp_func}).

\begin{lemma}\label{lemma:thm 1.1 Kim2012}
Consider QP (\ref{eq:qp_func}). If the objective function is bounded from below over a nonempty feasible domain, then it has a solution.
\end{lemma}
Detailed proofs of Lemma \ref{lemma:thm 1.1 Kim2012} can be found in  \cite{kim2012solution} and hence is omitted here.

Our proposed optimization models in this paper are convex and comprise $\ell_1$ norms in their objective functions. Since the $\ell_1$ norm is not differentiable everywhere, we  present some concepts and results regarding non-smooth convex optimization problems. For a function $f(\bm x): \R^n \mapsto \R$, its subdifferential at $\bm x$, denoted by $\partial (f)|_{\bm x}$ is defined as
$$\partial (f)|_{\bm x} \, \triangleq \, \{\bm g \in \R^n \ |\ \bm g^{T}(\bm y-\bm x) \leqslant f(\bm y) - f(\bm x), \, \forall \,\bm y\}. $$ It is known that $\partial f(\bm x)$ is a closed convex set (possibly empty) in general. But this set is nonempty and bounded for an interior point $\bm x $ in the domain of a convex function.
We are particularly interested in the $\ell_1$ norm as a function for which we have
$$\partial (\|\cdot\|_1)|_{\bm x} \, = \, J_1\times J_2 \times \cdots \times J_n,$$
with
$$J_k = \left\{
\begin{array}{ll}
[-1,1] & \mbox{if } \bm x_k = 0,\\
\{1\} & \mbox{if } \bm x_k > 0,\\
\{-1\} & \mbox{if } \bm x_k < 0,\\
\end{array}.\right.$$
Consider a convex optimization problem with possibly a non-smooth objective function as follows:
\begin{equation*}
\begin{array}{rl}
\min & f(\bm x) \\
\mbox{s.t.} & g_i(\bm x) \, \geqslant \, 0, \quad i=1, \dots, m.
\end{array}
\end{equation*}
Assuming that some certain constraint qualification holds at a feasible vector $\bm x^*$, this vector is a local optimal solution of this problem if and only if we have
\begin{equation}\label{eq:nonsmooth_opt_KKT}
\bm 0 \, \in \, \partial f(\bm x^*)-\sum_{i=1}^m \bm \alpha_i\nabla g_i(\bm x^*),
\end{equation}
for  $\bm \alpha \in \mathbb R^m_+$ such that $\bm \alpha_i g_i(\bm x^*)=0, \ i=1,\dots,m$; see Proposition 3.2.3, Theorem 6.1.8 and Exercise 5 of Section 7.2 in \cite{borwein2010convex}. The constraint qualification of interest in our discussion is the existence of an interior feasible point for the constraints. This is guaranteed based on linear or quadratic separability assumptions depending on the situation.


\section{Quadratic Surface Support Vector Machines} \label{sec:Quadratic Surface Support Vector Machines}
In this section, we introduce quadratic surface support vector machines (QSSVMs). We first describe the training data set and then discuss existing optimization models for QSSVMs. In the last subsection, we propose an $\ell_1$ norm regularized version of QSSVMs.
\subsection{Training Data Set}\label{subsec:notations}
For any classification problems, the training set is typically composed of $m$ samples each represented by a vector in $\R^n$ and a label. Mathematically, a training data set with two classes can be denoted by
\begin{equation}\label{def:dataset D}
 \mathcal D = \left\{\left(\dataxi, \datayi\right)_{i=1,\cdots,m} \ \left| \ \dataxi \, \in \, \R^n, \ \datayi \in \{-1, 1\}\right.
\right\},
\end{equation}
where $m$ is the sample size, $n$ is the number of features, $\dataxi = [\bm x^{(i)}_1, \dots, \bm x^{(i)}_n]^T \in \R^n$ is the vector of feature values for sample $i$, and $\datayi$ is the label for sample $i$. We let $\mathcal M^+ \triangleq \{i \,|\,\datayi=1\}$ and $\mathcal M^- \triangleq\{i\,|\,\datayi = -1\}$. We assume that $\mathcal M^+ \neq \emptyset$ and $\mathcal M^- \neq \emptyset$.
With the given training data, a recent approach proposed by Dagher \cite{dagher2008quadratic} and developed by Luo et al. \cite{luo2016soft}  aims to separate the data using a quadratic function: $$f(\bm x) \, = \, \frac{1}{2} \bm x^T \Bf W \bm x +\bm b^T \bm x + c,$$ with $\Bf W \in \mathbb S_n, \bm b\in \R^n,$ and $c \in \R$. To facilitate the presentation, we define a few matrices and vectors. Define the sample matrix as:
$$\Bf  X \, = \, \left[
\begin{array}{c}
\left(\bm x^{(1)}\right)^T\\
\vdots\\
\left(\bm x^{(m)}\right)^T\\
\end{array}
\right] \in \R^{m\times n}.$$
%
For the sake of analysis, we assume that:

\vspace{3mm}

\begin{itemize}
\item[(A1)] \centering \quad $\Bf X$ has full column rank.
\end{itemize}

\vspace{3mm}

Note that assumption (A1) is a very mild assumption. First, we typically have $m \gg n$. Secondly, if (A1) does not hold, then there is redundancy in the set of features. We typically remove the redundancy by conducting principal component analysis.

We say a data set $\mathcal D$ is \emph{quadratically separable} \cite{luo2016soft} if there exist $\Bf W\in \R^{n\times n},    \bm b \in \R^n,\text{ and } c \in \R$ such that
\begin{equation}\label{eq:quad_separable}
\begin{array}{c}
\displaystyle{ \frac{1}{2} {\dataxi}^T \Bf W \dataxi + \bm b^T \dataxi + c \,  > \, 0 }, \quad \forall \, i \in \mathcal M^+, \\[5pt]
\displaystyle{ \frac{1}{2} {\dataxi}^T \Bf  W \dataxi + \bm b^T \dataxi + c \, < \, 0}, \quad \forall \, i \in \mathcal M^-.
\end{array}
\end{equation}
We say a data set $\mathcal D$ is \emph{linearly separable} \cite{deng2012support} if there exist $\bm 0_n \neq \bm b \in \R^n,\text{ and }c \in \R$ such that
\begin{equation}\label{eq:linear_separable}
\begin{array}{c}
\bm b^T \dataxi + c > 0, \quad \forall \, i \in \mathcal M^+, \\[5pt]
\bm b^T \dataxi + c < 0, \quad \forall \, i \in \mathcal M^-.
\end{array}
\end{equation}
Note that a linearly separable data set is simply quadratically separable with  $\Bf W=\Bf 0$. To develop optimization models, we introduce the following definitions.
\begin{definition}[Vectors $\bm s$ and $\bm r$]\label{def: datasi and datari}
  For all $i = 1,\cdots,m $, define $\datasi \in \R^{\nct}$ as
  \newline
 \resizebox{0.89\textwidth}{!}{
\begin{minipage}{\linewidth}
\begin{equation*}
 \datasi = \left[\frac{1}{2} \dataxi_1 \dataxi_1, \dataxi_1 \dataxi_2, \dots, \dataxi_1 \dataxi_n, \frac{1}{2} \dataxi_2 \dataxi_2, \dataxi_2 \dataxi_3, \dots, \dataxi_2 \dataxi_n, \dots, \frac{1}{2} \dataxi_n \dataxi_n  \right]^T,
\end{equation*}
\end{minipage}
}
\newline
and
  $$
  \datari = \begin{bmatrix}
              \datasi \\
              \dataxi
            \end{bmatrix}.
  $$
\end{definition}
For convenience, denote $\bm w = \hv(\Bf W)$ for future use.
\begin{definition}[Vectorized parameters $\bm z$]\label{def: vectorized z}
  Define $\bm z \in \R^{\nct + n}$ as
  $
\bm z \, = \, \left[
\begin{array}{c}
  \bm w \\
  \bm b
\end{array}
\right].
  $
\end{definition}
\noindent Hence, by definitions \ref{def: vectorized z} and \ref{def: datasi and datari} we have
\begin{equation*}\label{constraint-linear-z}
  \frac{1}{2} {\dataxi}^T  \Bf W \dataxi + {\dataxi}^T \bm b + c = \bm z^T \datari + c.
\end{equation*}
Let
 $\Bf V:= \begin{bmatrix}
        \Bf I_{\frac{n(n+1)}{2}} & \Bf 0_{\frac{n(n+1)}{2}\times n}
      \end{bmatrix}.$
This implies that $\hv(\Bf W)=\bm w  = \Bf V\bm z \label{Vz}$.
\begin{definition}\label{def:X^i}
  For all $ i =1,\cdots,m$, define $\dataXi \in \R^{n \times n^2}$ as
  $$\dataXi = \mathbf I_n \otimes \left(\dataxi\right)^T = \left[
  \begin{array}{ccc}
    \left(\dataxi\right)^T &  &  \\
     & \ddots &  \\
     &  & \left(\dataxi\right)^T
  \end{array} \right].$$
\end{definition}
\begin{definition}[Matrix $\dataMi$]\label{def:M^i}
  For all $i =1,\cdots,m$, define $\dataMi \in \R^{n \times \nct }$ as
  $
  \dataMi \, = \,  \dataXi \Bf D_n.
  $
\end{definition}
\begin{definition}[Matrix $\dataHi$]\label{def: H^i}
  For all $i =1,\cdots,m$, define $\dataHi \in \R^{n \times \frac{n(n+3)}{2} }$ as
  $
\dataHi \, = \, \left[
\begin{array}{c;{2pt/2pt}c}
  \dataMi & \mathbf I_n
\end{array}
\right].
  $
\end{definition}
According to the above definitions, we have
$$\Bf W \dataxi \, = \, \dataXi \ve(\Bf W) \, = \, \dataXi \Bf  D_n \hv(\Bf W) \, = \, \dataMi \hv(\Bf W) \, = \, \dataMi \bm w,$$
and thus:
$$\Bf W \dataxi + \bm b \, = \, \dataMi \bm w + \mathbf I_n \bm b \, = \, \dataHi \bm z. $$
It follows that:
\begin{equation*}
\sum_{i=1}^m \left\|\Bf W \dataxi + \bm  b\right\|_2^2 \, = \, \sum_{i=1}^m \left(\dataHi \bm z\right)^T\left(\dataHi \bm z\right) \,= \, \bm z^T \left[\sum_{i=1}^{m} \left(\dataHi\right)^T \dataHi\right] \bm z.
\end{equation*}
Define
\begin{equation}\label{matrix G}
 \Bf  G := 2 \sum_{i=1}^{m} \left(\dataHi\right)^T \dataHi,
\end{equation}
which implies that $$ \sum_{i=1}^m \left\|\Bf W \dataxi + \bm b\right\|_2^2 = \frac{1}{2} \bm z^T \Bf G \bm z.$$
Clearly, we know that $\Bf  G \succeq 0$. We will provide conditions under which $\Bf  G \succ 0$ in a later section.

\subsection{Quadratic Surface SVM  Models}
Mathematical computations for the margin of quadratic surface  support vector machine  suggest the following formulation:
\begin{equation}\tag{QSSVM}\label{formulation:QSSVM}
\begin{aligned}
\min_{\Bf W, \bm b, c} \quad & \sum_{i=1}^m \|\Bf W \dataxi + \bm b\|_2^2 \\
s.t. \quad & \datayi\lp \frac{1}{2} {\dataxi}^T \Bf W \dataxi + {\dataxi}^T \bm b + c \rp \geqslant 1, \quad  i=1,\cdots,m, \\
           & \Bf W\in \mathbb S_n, \bm b \in \R^n, c \in \R.
\end{aligned}
\end{equation}
To account for possible noise and outliers in the data,  the following soft margin version of QSSVM penalizes mis-classifications:
\begin{equation}\tag{SQSSVM}\label{formulation:SQSSVM}
\begin{aligned}
\min_{\Bf W, \bm b, c, \bm \xi} \quad & \sum_{i=1}^m \|\Bf W \dataxi + \bm b\|_2^2 + \mu \sum_{i=1}^m  \bm \xi_i \\
s.t. \quad & \datayi\lp \frac{1}{2} {\dataxi}^T \Bf W \dataxi + {\dataxi}^T \bm b + c \rp \geqslant 1-  \bm \xi_i, \quad i=1,\cdots,m, \\
           & \Bf  W\in \mathbb S_n, \, \bm b \in \R^n, \, c \in \R, \, \bm \xi \in \R^m_+.
\end{aligned}
\end{equation}
With the notations defined in Subsection~\ref{subsec:notations}, we have the following equivalent formulation for \ref{formulation:QSSVM}:
\begin{equation} \tag{QSSVM\textprime}\label{formulation:QSSVM'}
       \begin{aligned}
          \min_{\bm z, c} \quad &  \frac{1}{2} {\bm z}^T \Bf G {\bm z}  \\
          s.t. \quad & \datayi\lp \bm z^T \datari + c \rp \geqslant 1, \quad i = 1, \dots, m, \\
           & \bm z \in \R^{\frac{n(n+1)}{2}+n}, \, c \in \R,
       \end{aligned}
\end{equation}
and similarly for \ref{formulation:SQSSVM}:
      \begin{equation} \tag{SQSSVM\textprime}\label{formulation:SQSSVM'}
       \begin{aligned}
          \min_{\bm z, c, \bm \xi} \quad &  \frac{1}{2} {\bm z}^T \Bf G {\bm z}  +  \mu \sum_{i=1}^m  \bm \xi_i \\
          s.t. \quad & \datayi\lp \bm z^T \datari + c \rp \geqslant 1-  \bm \xi_i, \quad i = 1, \dots, m \\
           & \bm z \in \R^{\frac{n(n+1)}{2}+n}, \, c \in \R, \, \bm \xi \in \R^m_+.
       \end{aligned}
\end{equation}
\subsection{$\ell_1$ Norm Regularized Quadratic Surface SVM  Models}  \label{L1 norm regularized models}
As we can see, by using quadratic surface to separate the data, the classification model complexity is increased significantly. While the complexity provides extra flexibility to separate data in the training set, it could also lead to over-fitting issues and hence inferior classification performance on testing data sets. One particular situation is when the training data set is linearly separable, it is desirable that QSSVMs can find a hyperplane to separate the data. However, there is no guarantee if we use models (\ref{formulation:QSSVM}) or (\ref{formulation:SQSSVM}). On the other hand, $\ell_1$ norm regularization technique has been shown to reduce model complexity in many applications. Therefore, we propose to introduce L1-regularization into QSSVMs:
\begin{equation} \tag{L1-QSSVM}\label{formulation:QSSVM-vecl1}
\begin{aligned}
\min_{\Bf W, \bm b, c} \quad & \sum_{i=1}^m \|\Bf W \dataxi + \bm b\|_2^2 + \lambda \sum_{1\leqslant i\leqslant j\leqslant n} | \Bf{W}_{ij}| \\
s.t. \quad & \datayi\lp \frac{1}{2} {\dataxi}^T \Bf W \dataxi + {\dataxi}^T \bm b + c \rp \geqslant 1, \quad i = 1, \dots, m, \\
           & \Bf W\in \mathbb S_n, \, \bm b \in \R^n, \, c \in \R,
\end{aligned}
\end{equation}
where $\lambda$ is a positive constant that penalizes nonzero elements of the model matrix $\Bf W$. This is equivalent to:
\begin{equation} \tag{L1-QSSVM\textprime}\label{formulation:QSSVM-vecl1-vec}
       \begin{aligned}
          \min_{\bm z,c} \quad &  \frac{1}{2} {\bm z}^T \Bf  G {\bm z}  + \lambda \|\Bf V\bm z\|_1\\
          s.t. \quad & \datayi\lp \bm z^T \datari + c \rp \geqslant 1, \quad i = 1, \dots, m, \\
           & \bm z \in \R^{\frac{n(n+1)}{2}+n}, \, c \in \R.
       \end{aligned}
\end{equation}

To account for outliers, we consider the following soft version of \ref{formulation:QSSVM-vecl1}:
\begin{equation} \tag{L1-SQSSVM}\label{formulation:SQSSVM-vecL1}
\begin{aligned}
\min_{\Bf W, \bm b, c, \bm \xi} \quad & \sum_{i=1}^m \|\Bf W \dataxi + \bm b\|_2^2 + \lambda \sum_{1\leqslant i\leqslant j\leqslant n} | \Bf W_{ij}| + \mu \sum_{i=1}^m  \bm \xi_i \\
s.t. \quad & \datayi\lp \frac{1}{2} {\dataxi}^T \Bf W \dataxi + {\dataxi}^T \bm b + c \rp \geqslant 1-  \bm \xi_i, \quad i = 1, \dots, m, \\
           & \Bf W\in \mathbb S_n, \, \bm b \in \R^n, \,  c \in \R, \, \bm \xi \in \R^m_+,
\end{aligned}
\end{equation}
where $\mu>0$ is a positive penalty for incorporating noise and outliers.
This problem can be equivalently rewritten as the following convex program:
      \begin{equation} \tag{L1-SQSSVM\textprime}\label{formulation:SQSSVM-vectorL1-vec}
       \begin{aligned}
          \min_{\bm z, c, \bm \xi} \quad &  \frac{1}{2} {\bm z}^T \Bf G {\bm z} \, + \, \lambda \|\Bf V \bm z\|_1 \, +  \, \mu \sum_{i=1}^m  \bm \xi_i \\
          s.t. \quad & \datayi\lp \bm z^T \datari + c \rp \geqslant 1-  \bm \xi_i, \quad i = 1, \dots, m, \\
           & \bm z \in \R^{\frac{n(n+1)}{2}+n}, \, c \in \R, \, \bm \xi \in \R^m_+.
       \end{aligned}
      \end{equation}
In the next section, we study theoretical properties of the proposed models.
\section{Theoretical Properties of $\ell_1$ Norm Regularized QSSVMs} \label{sec:Theoretical Properties of L1 Norm Regularized QSSVMs}
In this section, we explore theoretical properties of the proposed $\ell_1$ norm regularized QSSVM models. We first discuss the solution existence and uniqueness of our proposed models. We then study the soft margin models and show that the margin vanishes when $\mu$ is large enough and the data is separable. Lastly, we examine the effects of the $\ell_1$ norm regularization.
\subsection{Solution Existence and Uniqueness}
\begin{theorem}[Solution Existence] \label{theorem:existence}
   Given any data set $\mathcal D$ defined  in (\ref{def:dataset D}), the model \ref{formulation:SQSSVM-vecL1} has an optimal solution with a finite objective value. A similar result holds for (\ref{formulation:QSSVM-vecl1}) on any linearly or qaudratically separable data set.
\end{theorem}
\begin{proof}
The model (\ref{formulation:SQSSVM-vecL1}) is equivalent to  (\ref{formulation:SQSSVM-vectorL1-vec}), which reduces to a convex quadratic program with linear constraints by a standard technique. Given a  data set $\mathcal D$, the feasible set is nonempty for an arbitrary $\bm z$ and $c$ with  $$\bm \xi_i=\max\left[0, 1- \datayi\lp \bm z^T \datari + c \rp\right],  \quad i = 1, \dots, m. $$ Further, the objective function in (\ref{formulation:SQSSVM-vectorL1-vec}) is bounded below by zero over the feasible set. Hence, Lemma \ref{lemma:thm 1.1 Kim2012} implies that (\ref{formulation:SQSSVM-vectorL1-vec}) has an optimal solution with a finite objective value and so does (\ref{formulation:SQSSVM-vectorL1-vec}). The same idea applies to the model (\ref{formulation:QSSVM-vecl1-vec}). The second statement in the theorem follows from a similar argument.

\end{proof}

We next show that if $\Bf G$ is positive definite, then $\bm z^*$ must be unique.  One can prove this using tedious rewriting of (\ref{formulation:SQSSVM-vectorL1-vec}) as a standard quadratic program and then exploiting  gradient uniqueness, nonetheless we provide a direct proof below.
\begin{theorem}[$\bm z$-Uniqueness]  \label{theorem:z_uniqueness}
Assume that the matrix $\Bf  G$ defined  in (\ref{eq:G_definition}) is positive definite. Then, the solution $\bm z^*$ is unique for the model (\ref{formulation:SQSSVM-vectorL1-vec}).
\end{theorem}
\begin{proof}
Assume that $(\bm z^*, c^*,\bm \xi^*)$ and $(\bm {\bar z}, \bar c,\bm {\bar \xi})$ are two optima of (\ref{formulation:SQSSVM-vectorL1-vec}) such that $\bm z^*\ne \bm {\bar z}$. Since this problem is convex, its optimal solution set is convex. Thus, any convex combination of $(\bm z^*, \bm c^*,\bm \xi^*)$ and $(\bm  {\bar z}, \bar c,\bm {\bar \xi})$ results in the same optimal value. Let
\[
 q(\bm z,\bm \xi):=\frac{1}{2} {\bm z}^T \Bf G {\bm z} \, + \, \lambda \|\Bf V \bm z\|_1 \, +  \, \mu \sum_{i=1}^m  \bm \xi_i.
\]
 Thus, for $\delta\in (0,1)$, we have $$q\left(\delta \bm z^*+(1-\delta) \bm  {\bar z}, \delta \bm \xi^*+(1-\delta) \bm {\bar \xi}\right)=\delta q(\bm z^*, \bm \xi^*)+(1-\delta)q(\bm  {\bar z}, \bm {\bar \xi}).$$ Equivalently,

\resizebox{0.86\textwidth}{!}{
\begin{minipage}{\linewidth}
\begin{equation*}
 \frac{1}{2} ({ \delta \bm z^*+(1-\delta) \bm {\bar z}})^T \Bf  G ({\delta \bm z^*+(1-\delta)\bm  {\bar z}})  +  \lambda \| \delta \Bf V\bm z^*+(1-\delta)\Bf V \bm {\bar z}\|_1  +  \mu \sum_{i=1}^m (\delta  \bm \xi^*_i+(1-\delta) \bm {\bar \xi}_i),
\end{equation*}
\end{minipage}
}
\newline
is equal to 
\newline
\resizebox{0.85\textwidth}{!}{
\begin{minipage}{\linewidth}
\begin{equation*}
 \frac{\delta}{2} {\bm z^*}^T \Bf  G \bm z^* + \lambda \|\delta \Bf V \bm z^*\|_1 +  \mu \sum_{i=1}^m \delta  \bm \xi^*_i
+
\frac{(1-\delta)}{2} {\bm  {\bar z}}^T \Bf  G \bm  {\bar z} + \lambda \|(1-\delta) \Bf V \bm  {\bar z}\|_1 +  \mu \sum_{i=1}^m (1-\delta)  \bm {\bar \xi}_i.
\end{equation*}
\end{minipage}
}
\newline
By simple calculations, we get
\newline
\resizebox{0.89\textwidth}{!}{
\begin{minipage}{\linewidth}
\begin{equation*}
\frac{\delta^2-\delta}{2}\Bigg[{\bm z^*}^T \Bf G \bm z^{*}-2{\bm z^*}^T \Bf  G \bm  {\bar z}+{\bm  {\bar z}}^T \Bf  G \bm  {\bar z}\Bigg]+
\lambda \Bigg[ \| \delta \Bf V\bm z^*+(1-\delta)\Bf V \bm  {\bar z}\|_1-  \|\delta \Bf V \bm z^*\|_1 -    \|(1-\delta) \Bf V \bm {\bar z}\|_1\Bigg]=0,
\end{equation*}
\end{minipage}
}
\newline
which further means that
\newline
\resizebox{0.89\textwidth}{!}{
\begin{minipage}{\linewidth}
\begin{equation*}
\frac{\delta-\delta^2}{2}\Bigg[({\bm z^*-\bm  {\bar z}})^T \Bf  G ({\bm z^*-\bm  {\bar z}})\Bigg]+
\lambda \Bigg[
\|\delta \Bf V \bm z^*\|_1+    \|(1-\delta) \Bf V \bm z^*\|_1
-\| \delta \Bf V\bm z^*+(1-\delta)\Bf V \bm  {\bar z})\|_1
\Bigg]
=0.
\end{equation*}
\end{minipage}
}
\newline
Since $\delta\in (0,1)$ and $\lambda>0$,  the  positive definiteness of $\Bf G$ along with the Cauchy-Schwartz inequality imply that $({\bm z^*-\bm  {\bar z}})^T \Bf G ({\bm z^*-\bm  {\bar z}})=0$ and $ \| \delta \Bf V\bm z^*+(1-\delta)\Bf V \bm  {\bar z})\|_1=\|\delta \Bf V \bm z^*\|_1+    \|(1-\delta) \Bf V \bm {\bar z}\|_1.$ The former equation yields $\bm z^*=\bm  {\bar z}$.
\end{proof}
We therefore investigate the conditions under which the matrix  $\Bf  G$ is positive definite below.
\begin{theorem}[Positive-Definiteness of $\Bf G$]\label{th:pd_G}
Let (A1) hold. Also, suppose that

\vspace{3mm}

\begin{itemize}
  \item[(A2)] \centering \quad $\mathbf 1_m$ is not in the column space of $\Bf X$.
\end{itemize}

\vspace{3mm}

Then, the matrix $\Bf G$ defined in (\ref{eq:G_definition}) is positive definite.
\end{theorem}
\begin{proof}
The proof can be found in Appendix \ref{proof:thm:pd_G}.

\end{proof}
We next show that the set of $\Bf X$'s for which the associated matrix $\Bf  G$ is not positive definite is of Lebesgue measure zero in $\R^{m\times n}$. To be more specific, we define the following set:
\begin{equation}\label{eq:pd_data_set}
S \, \triangleq \, \{ \Bf  X \in \R^{m\times n}\  | \  \Bf {G} \mbox{ is positive definite}\}.
\end{equation}
\begin{theorem} \label{th:surely_pd_G}
Let $m\ge n+1$. The complement of $S$ defined in Equation (\ref{eq:pd_data_set}) is of Lebesgue measure zero in $\R^{m\times n}$.
\end{theorem}
\begin{proof}
The proof can be found in Appendix \ref{proof:thm:surely_pd_G}.
\end{proof}

While the uniqueness of $c^*$ and $\bm\xi^*$ in \ref{formulation:SQSSVM-vectorL1-vec} is generally not guaranteed, we have the following observation. Let $(\bm z^*, \bm \xi^*, c^*)$ be an optimal solution of (\ref{formulation:SQSSVM-vectorL1-vec}), it is clear that when $\bm z^*$ is given, $(c^*, \bm \xi^*)$ must solve the following linear program.
      \begin{equation}\label{eq:z-xi-LP}
       \begin{aligned}
          (c^*, \bm \xi^*) \in \argmin_{c, \bm \xi} \quad &  \sum_{i=1}^m \bm \xi_i \\
          s.t. \quad & \datayi\lp f_i + c \rp \geqslant 1-  \bm \xi_i, \quad i = 1, \dots, m, \\
           & c \in \R, \, \bm \xi \in \R^m_+,
       \end{aligned}
      \end{equation}
where $f_i = \left(\bm z^*\right)^T \datari$. The uniqueness of $c^*$ and $\bm \xi^*$ is therefore related to the solution uniqueness of linear program (\ref{eq:z-xi-LP}). The readers can refer to \cite{mangasarian1979uniqueness}, for results regarding solution uniqueness of linear programs. Next, we provide some conditions  under which the solution of \ref{formulation:SQSSVM-vecL1} is unique.

\subsection{Vanishing Margin $\bm \xi$ in Quadratically Separable Case}

As we mentioned, the training data may contain noise, and hence not separable. We include soft margin $\bm \xi$ to handle this situation. In principle, when the data is separable, the soft margin should vanish in the solution. In this subsection, we show that $\bm \xi$ does vanish when the data is quadratically separable, if $\mu$ is large enough. For convenience, we consider the following equivalent formulations for (\ref{formulation:QSSVM-vecl1}) and (\ref{formulation:SQSSVM-vecL1}).

\begin{equation} \tag{L1-QSSVM\textprime\textprime}\label{formulation: L1-QSSVM-vecl1-wbc}
       \begin{aligned}
          \min_{\bm w, \bm b, c} \quad &  q(\bm w, \bm b, c) \, \triangleq \,  \sum_{i=1}^m {\bm w}^T {\dataMi}^T \dataMi {\bm w} + 2 \sum_{i=1}^m {\bm w}^T {\dataMi}^T  {\bm b}  + m \bm b^T \bm b + \lambda \|\bm w\|_1 \\
          s.t. \quad & \datayi\lp \bm w^T \datasi + \bm b^T \dataxi + c \rp \, \geqslant \, 1, \quad i = 1, \dots, m, \\
           & \bm w \in \R^{\frac{n(n+1)}{2}}, \, \bm b \in \R^n, \, c \in \R.
       \end{aligned}
\end{equation}

\begin{equation} \tag{L1-SQSSVM\textprime\textprime}\label{formulation: L1-SQSSVM-vecl1-wbc}
       \begin{aligned}
          \min_{\bm w, \bm b, c, \bm \xi} \quad &  \sum_{i=1}^m {\bm w}^T {\dataMi}^T \dataMi {\bm w} + 2 \sum_{i=1}^m {\bm w}^T {\dataMi}^T  {\bm b}  + m \bm b^T \bm b + \lambda \|\bm w\|_1 + \mu \sum_{i=1}^m  \bm \xi_i\\
          s.t. \quad & \datayi\lp \bm w^T \datasi + \bm b^T \dataxi + c \rp \, \geqslant \, 1 - \bm \xi_i, \quad i = 1, \dots, m, \\
           & \bm w \in \R^{\frac{n(n+1)}{2}}, \, \bm b \in \R^n, \, c \in \R, \, \bm \xi \in \R^{m}_+.
       \end{aligned}
\end{equation}

It is clear that when the training data is quadratically separable, the hard margin model (\ref{formulation: L1-QSSVM-vecl1-wbc}) is feasible and has an optimal solution $(\bm w^*, \bm b^*, c^*)$. Therefore, according to (\ref{eq:nonsmooth_opt_KKT}), there exists a multiplier vector $\bm \alpha^* \in \R^m$ so that $(\bm w^*, \bm b^*, c^*, \bm \alpha^*)$ satisfies the following KKT conditions:
 \begin{equation}\label{KKT_L1_QSSVM}
   \left\{
     \begin{aligned}
      &\bm  0 \in  2 \sum_{i=1}^m {\dataMi}^T \dataMi \bm w^* + 2 \sum_{i=1}^m {\dataMi}^T \bm b^*  + \lambda  \partial\left.( \|\cdot\|_1)\right|_{\bm w^*} - \sum_{i=1}^m  \bm \alpha_i^* \datayi \datasi,\\
      & 2\sum_{i=1}^m \dataMi \bm w^* + 2m \bm b^* = \sum_{i=1}^{m}  \bm \alpha_i^* \datayi \dataxi, \\
      & \sum_{i=1}^{m} \bm \alpha_i^* \datayi = 0, \\
      &  \bm \alpha^*_i \lp 1-\datayi \lp {\bm w^*}^T \datasi + {\bm b^*}^T \dataxi + c^* \rp \rp = 0, \quad  i = 1, \dots, m, \\
      & \bm \alpha^*\geqslant \bm 0 \\
      & 1-\datayi \lp {\bm w^*}^T \datasi + {\bm b^*}^T \dataxi + c^* \rp \leqslant 0,
      \quad  i = 1, \dots, m .\\
     \end{aligned}
   \right.
   \end{equation}
Note that (\ref{eq:quad_separable}) implies that (\ref{formulation: L1-QSSVM-vecl1-wbc}) has strictly feasible solutions, i.e., there exists $(\overline{\bm w}, \overline{\bm b}, \overline c)$ such that
$$\datayi\lp \overline{\bm w}^T \datasi + \overline{\bm b}^T \dataxi + \overline c \rp \, > \, 1, \quad i = 1, \dots, m.$$
Therefore, by Exercise 5.3.1 in \cite{bertsekas1997nonlinear}, we know that the set of Lagrangian multipliers $\alpha^*$ is bounded from above. In particular, we have
\begin{equation}\label{eq:multiplier_upper_bound}
\|\bm \alpha^*\|_1 \, \leqslant \, \pi \, \triangleq \, \frac{q(\overline{\bm w},\overline{\bm b}, \overline c)-q^*}{c_1},
\end{equation}
where $q^*$ is the optimal value of (\ref{formulation: L1-QSSVM-vecl1-wbc}) and $$c_1 = \displaystyle{\min_{i = 1, \dots, m} } \bigg[\datayi \lp \overline{ \bm w}^T \datasi + \overline { \bm b}^T \dataxi + \overline c\rp-1\bigg].$$

\begin{theorem}  \label{theorem:xi_0_mu_large}
Suppose that the data set $\mathcal D$ is quadratically separable and the matrix $\Bf G$ is positive definite. For any $\lambda$, there exists a $\underline{\mu}$ (depending on $\lambda$), such that for all $\mu>\underline{\mu}$, (\ref{formulation: L1-SQSSVM-vecl1-wbc}) has a unique solution $(\bm w^*, \bm b^*, c^*, \bm \xi^*)$ with $\bm \xi^*=\bm 0_m$.
\end{theorem}
\begin{proof}
We start with an optimal solution of (\ref{formulation: L1-QSSVM-vecl1-wbc}), say $(\bm w^*, \bm b^*, c^*)$ and then construct an optimal solution of (\ref{formulation: L1-SQSSVM-vecl1-wbc}). It is clear that the KKT conditions for convex problem (\ref{formulation: L1-SQSSVM-vecl1-wbc}) yields the necessary and sufficient conditions for optimality. Let $\bm \alpha^*$ be a vector of multipliers such that $(\bm w^*, \bm b^*, c^*, \bm \alpha^*)$ satisfies (\ref{KKT_L1_QSSVM}). Let
$$\underline \mu > \pi,$$ where $\pi$ is defined in (\ref{eq:multiplier_upper_bound}). Notice that $\pi$ depends on $\lambda$. It is clear that $\underline \mu > \| \bm \alpha^*\|_\infty$. For any $\mu> \underline \mu$, we let $\bm \eta^* = \mu \bm 1_m - \bm \alpha^*$. By the definition of $\underline \mu$, it is clear that $\bm \eta^*\geqslant 0$. It suffices to verify that $(\bm w^*, \bm b^*, c^*, \bm \xi^* = \bm 0_m, \bm \alpha^*, \bm \eta^*)$ satisfies the following KKT conditions:
 \begin{equation}\label{KKT_L1_SQSSVM}
   \left\{
     \begin{aligned}
      &\bm  0 \in  2 \sum_{i=1}^m {\dataMi}^T \dataMi \bm w^* + 2 \sum_{i=1}^m {\dataMi}^T \bm b^*  + \lambda  \partial\left.( \|\cdot\|_1)\right|_{\bm w^*} - \sum_{i=1}^m  \bm \alpha_i^* \datayi \datasi, \\
      & 2\sum_{i=1}^m \dataMi \bm w^* + 2m \bm b^* = \sum_{i=1}^{m} \bm \alpha_i^* \datayi \dataxi, \\
      & \sum_{i=1}^{m}   \bm \alpha_i^* \datayi = 0, \\
      & \bm \alpha^*_i \lp 1- \bm \xi_i^* - \datayi \lp {\bm w^*}^T \datasi + {\bm b^*}^T \dataxi + c^* \rp \rp = 0, \quad i = 1, \dots, m, \\
      & \bm \alpha^*\geqslant \bm 0, \\
      & 1- \bm \xi_i^* - \datayi \lp {\bm w^*}^T \datasi + {\bm b^*}^T \dataxi + c^* \rp \leqslant 0,
      \quad  i = 1, \dots, m,\\
    &  \bm \xi^*_i \bm \eta^*_i=0   \quad i = 1, \dots, m, \\
     & \bm \xi^*\geqslant \bm 0, \\
     & \bm\eta^*\geqslant \bm 0, \\
     & \mu = \bm \alpha^*_i + \bm \eta^*_i,  \quad  i = 1, \dots, m.
     \end{aligned}
   \right.
   \end{equation}
Since $(\bm w^*, \bm b^*, c^*, \bm \alpha^*)$ satisfies (\ref{KKT_L1_QSSVM}), it is straightforward to verify that $(\bm w^*, \bm b^*, c^*, \bm \xi^* =\bm 0_m, \bm \alpha^*, \bm \eta^*)$ satisfies (\ref{KKT_L1_SQSSVM}) and hence is an optimal solution of \ref{formulation: L1-SQSSVM-vecl1-wbc}. Since $\Bf G$ is positive definite, we know that $\bm w^*$ and $\bm b^*$ are unique by Theorem~\ref{theorem:z_uniqueness}. And since $(c^*, \bm \xi^*)$ must be an optimal solution of linear program (\ref{eq:z-xi-LP}), we conclude that $\bm \xi^*=\bm 0_m$ is the only solution. By an argument similar to Theorem 5 in \cite{luo2016soft}, $c^*$ is also unique. This concludes the proof.
\end{proof}

\begin{remark}
Theorem~\ref{theorem:xi_0_mu_large} can be extended to the original SQSSVM models studied in \cite{luo2016soft}, leading to a result stronger than Theorem 2 therein for which the parameter $\mu$ must go to infinity.
\end{remark}

\subsection{Effects of $\ell_1$ Norm Regularization}
We study next the effects of $\ell_1$ norm regularization. First, we consider the case when the training data set $\mathcal D$ defined in  (\ref{def:dataset D}) is linearly separable according to (\ref{eq:linear_separable}). As we have discussed, in this case, it is desirable that the \ref{formulation:QSSVM} returns a separation hyperplane rather than a quadratic surface, but there is no guarantee of such property for the model. On the other hand, our proposed model (\ref{formulation:QSSVM-vecl1}) captures this desired property for a finite  but large enough $\lambda$. Mathematically, this is equivalent to $\bm w^*=\bm 0$ in an optimal solution of (\ref{formulation:QSSVM-vecl1}) for some $\lambda$. To show this, we start with an optimal solution of the standard SVM. Consider the following standard SVM with hard margin:
\begin{equation} \tag{SVM}\label{formulation: SVM-hard margin}
 \begin{aligned}
  \min_{\bm u, d} \quad & \frac{1}{2} \|\bm u\|^2_2 \\
  s.t. \quad & \datayi \lp \bm u^T \dataxi + d \rp \geqslant 1, \quad i = 1, \dots, m, \\
  & \bm u \in \R^n, d \in \R.
 \end{aligned}
\end{equation}
Under the linear separability assumption (\ref{eq:linear_separable}), the standard SVM with hard margin (\ref{formulation: SVM-hard margin}) is feasible, and the objective function is bounded from below. Therefore, there exists an optimal solution $(\bm u^*, d^*)$ that solves (\ref{formulation: SVM-hard margin}).  The Lagrangian of (\ref{formulation: SVM-hard margin}) is the following:
 \begin{equation}\label{L_SVM}
   \mathcal {L_{\text{SVM}}} \lp \bm u, d, \bm \beta \rp = \frac{1}{2} \|\bm u\|_2^2 + \sum_{i=1}^m  \bm \beta_i \lp 1-\datayi \lp \bm u^T \dataxi + d \rp  \rp,
 \end{equation}
where $\bm \beta = [ \bm \beta_1, \bm \beta_2, \dots, \bm \beta_m]^T$ are the Lagrangian multipliers. It is clear that the necessary and sufficient condition for $(\bm u^*, d^*)$ to be an optimal solution of (\ref{formulation: SVM-hard margin}) is the existence of a vector of Lagrangian multipliers $\bm \beta^*$ such that the following KKT conditions are satisfied:
 \begin{equation}\label{KKT:SVMhardmargin}
 \left\{
   \begin{aligned}
     & \bm u^* \, = \, \sum_{i=1}^m \bm \beta^*_i \datayi \dataxi, \\
     & \sum_{i=1}^m \bm \beta^*_i \datayi \, = \, 0, \\
     &  \bm \beta^*_i \lp 1-\datayi \lp {\bm u^*}^T \dataxi + d^* \rp  \rp \, = \, 0, \quad i = 1, \dots, m,\\
     &  \bm \beta^*_i \, \geqslant \, 0,\quad  \, i = 1, \dots, m, \\
     & 1-\datayi \lp {\bm u^*}^T \dataxi + d \rp \, \leqslant \, 0, \quad  \, i = 1, \dots, m.
   \end{aligned}
   \right.
 \end{equation}
In the next theorem, we show that $(\bm 0_{n\times n}, \bm u^*, d^*)$ in fact solves (\ref{formulation:QSSVM-vecl1}) when $\lambda$ is large enough.

\begin{theorem}[Equivalence of SVM and L1-QSSVM for  finite $\lambda$] \label{thm: equivalence SVM & QSSVM-vecl1 hardmargin}
Suppose that the training data set $\mathcal D$ defined in (\ref{def:dataset D}) is linearly separable (as defined in (\ref{eq:linear_separable})). Let $(\bm u^*, d^*, \bm \beta^*)$ satisfy the KKT conditions (\ref{KKT:SVMhardmargin}), then $\Bf W^*=\bm 0_{n\times n}$, $\bm b^* = \bm u^*$, $c^* = d^*$, solves (\ref{formulation:QSSVM-vecl1}) when $\lambda$ is large enough.
\end{theorem}
\begin{proof}
For convenience, we consider the equivalent formulation (\ref{formulation: L1-QSSVM-vecl1-wbc}).
  The Lagrangian of (\ref{formulation: L1-QSSVM-vecl1-wbc}) is the following:
  \begin{equation}\label{L_Q}
   \begin{aligned}
    \mathcal L_Q \lp \bm w, \bm b, c, \bm \alpha \rp = &  \sum_{i=1}^m \bm w^T {\dataMi}^T \dataMi \bm w + 2 \sum_{i=1}^m \bm w^T {\dataMi}^T \bm b + m \bm b^T \bm b + \lambda \|\bm w\|_1 + \\
    & \sum_{i=1}^m \bm \alpha_i \lp 1- \datayi\lp \bm w^T \datasi + \bm b^T \dataxi + c \rp \rp.
   \end{aligned}
  \end{equation}
The optimality KKT conditions for this problem based on  (\ref{eq:nonsmooth_opt_KKT}) are as follows:
 \begin{equation}\label{L_Q KKT}
   \left\{
     \begin{aligned}
      &\bm  0 \in  2 \sum_{i=1}^m {\dataMi}^T \dataMi \bm w^* + 2 \sum_{i=1}^m {\dataMi}^T \bm b^*  + \lambda  \partial\left.( \|\cdot\|_1)\right|_{\bm w^*} - \sum_{i=1}^m \bm \alpha_i^* \datayi \datasi, \\
      & 2\sum_{i=1}^m \dataMi \bm w^* + 2m \bm b^* = \sum_{i=1}^{m} \bm \alpha_i^* \datayi \dataxi, \\
      & \sum_{i=1}^{m}\bm \alpha_i^* \datayi = 0, \\
      & \bm \alpha^*_i \lp 1-\datayi \lp {\bm w^*}^T \datasi + {\bm b^*}^T \dataxi + c^* \rp \rp = 0, \quad  i = 1, \dots, m, \\
      & \bm \alpha^*_i \geqslant 0, \quad i = 1, \dots, m, \\
      & 1-\datayi \lp {\bm w^*}^T \datasi + {\bm b^*}^T \dataxi + c^* \rp \leqslant 0,
      \quad  i = 1, \dots, m.
     \end{aligned}
   \right.
   \end{equation}
From (\ref{KKT:SVMhardmargin}), one can see that $
     \begin{bmatrix}
     {\bm w^*}^T & {\bm b^*}^T & {c^*} & {\bm \alpha^*}^T
   \end{bmatrix} = \begin{bmatrix} \bm 0^T & {\bm u^*}^T & {d^*} & 2m {\bm \beta^*}^T
   \end{bmatrix}$
satisfies  all the equations in (\ref{L_Q KKT}) except
 $\bm 0 \in \frac{\partial \mathcal L_Q}{\partial \bm w}\biggr\rvert_{\bm 0}$. However, this condition is equivalent to \begin{equation}\label{lower_lambda}
     \left\|\sum_{i=1}^{m}{\dataMi}^T \bm u^* - m \sum_{i=1}^{m}\bm \beta^*_i \datayi \datasi\right\|_\infty \le \frac{\lambda}{2}.
 \end{equation}
So, if we choose $\lambda$ large enough, this condition also holds.
Hence, $\begin{bmatrix} \bm 0^T & {\bm u^*}^T & {d^*} & 2m {\bm \beta^*}^T
   \end{bmatrix}$ satisfies the KKT conditions when $\lambda$ is large enough. Since the KKT conditions are also sufficient for the solution of a convex program, we conclude that $\Bf W^*=\bm 0_{n\times n}$, $\bm b^* = \bm u^*$, $c^* = d^*$, solves (\ref{formulation:QSSVM-vecl1}).
\end{proof}

Next, we find a lower  bound for such $\lambda$. From the inequality (\ref{lower_lambda}), it is enough to find an upper bound for $$\left\|\sum_{i=1}^{m}{\dataMi}^T \bm u^* - m \sum_{i=1}^{m}\bm \beta^*_i \datayi \datasi\right\|_\infty,$$ where $\bm u^*$ and $ \bm \beta^*$ satisfy (\ref{KKT:SVMhardmargin}). Let $$\bm a = \sum_{i=1}^{m}{\dataMi}^T \bm u^* - m \sum_{i=1}^{m}\bm \beta^*_i \datayi \datasi \in \mathbb{R}^{n(n+1)/2}.$$
By denoting  $\Bf T:= \Bf X^T \text{Diag}(\bm y)$, we have $\bm u^* = \displaystyle{\sum_{i=1}^m }\bm \beta^*_i \datayi \dataxi = \Bf T \bm \beta^*$ so that $$\|\bm u^*\|_2\leqslant \|\Bf T\|_2\|\bm \beta^*\|_2=\|\Bf X^T\|_2\|\bm \beta^*\|_2=\|\Bf X\|_2\|\bm \beta^*\|_2\le \|\Bf X\|_2 \|\bm \beta^*\|_1.$$ Thus, along with definitions of  (\ref{def: datasi and datari}) and (\ref{def:M^i}), for $ k \in  \{1,2,\dots, n(n+1)/2\}$,  we have:
\begin{align*}
|\bm a_k| \ \leqslant &  \ \sum_{i=1}^{m} \|{\dataMi_{\bullet k}}\|_2 \|\bm u^*\|_2  \ + \  m \sum_{i=1}^{m}\bm \beta^*_i |\bm s^{(i)}_k|  \\
\leqslant & \ \sum_{i=1}^{m} \|\dataxi\|_2 \|\bm u^*\|_2  \ + \  m \sum_{i=1}^{m}\bm \beta^*_i  \|\dataxi\|_\infty^2 \\
\leqslant & \ m \|\bm u^*\|_2 \ \max_i \|\dataxi\|_2   \ + \ m  \|\bm \beta ^*\|_1 \  \max_i \|\dataxi\|_\infty^2 \\
\leqslant & \ m \|\Bf X\|_2 \|\bm \beta^*\|_1 \ \max_i \|\dataxi\|_2  \ + \  m  \|\bm \beta ^*\|_1\ \max_i \ \|\dataxi\|_\infty^2. \\
= & \ m\|\bm \beta^*\|_1\bigg( \|\Bf X\|_2 \ \max_i \|\dataxi\|_2  \ + \  \max_i \ \|\dataxi\|_\infty^2\bigg).
\end{align*}
It suffices to find an upper bound for $\|\bm \beta^*\|_1$. Note that (\ref{eq:linear_separable}) implies that (\ref{formulation: SVM-hard margin}) has strictly feasible solutions, i.e., there exists $( \overline{\bm u}, \overline d)$ such that
$$\datayi\lp \overline{\bm u}^T \dataxi + \overline d \rp \, > \, 1, \quad i = 1, \dots, m.$$
Therefore, by Exercise 5.3.1 in \cite{bertsekas1997nonlinear}, we know that the set of Lagrangian multipliers $\bm \beta^*$ is bounded from above. In particular, we have
\begin{equation*}
\|\bm \beta^*\|_1 \, \leqslant \, \frac{\|\overline {\bm u}\|_2^2-\|\bm u^*\|_2^2}{2c_2}
\leqslant \frac{\|\overline {\bm u}\|_2^2}{2c_2},
\end{equation*}
where  $c_2 = \displaystyle{\min_{i = 1, \dots, m}}  \bigg[\datayi \lp \overline{\bm u}^T \dataxi + \overline d\rp-1\bigg]$. Thus, to have the above result,  it suffices to have
\[
\lambda \geqslant m\frac{\|\overline {\bm u}\|_2^2}{2c_2}\bigg( \|\Bf X\|_2 \ \max_i \|\dataxi\|_2  \ + \  \max_i \|\dataxi\|_\infty^2\bigg).
\]

We can show a similar result for the soft margin version. We consider the soft margin version of SVM is as follows.
\begin{equation} \tag{SSVM}\label{formulation: SVM-soft margin}
 \begin{aligned}
  \min_{\bm u, d, \bm \xi} \quad & \frac{1}{2} \|\bm u\|^2_2 +\mu \sum_{i=1}^m \bm \xi_i\\
  s.t. \quad & \datayi \lp \bm u^T \dataxi + d \rp \geqslant 1 -  \bm \xi_i, \quad i = 1, \dots, m, \\
  & \bm u \in \R^n, \, d \in \R, \, \bm \xi\in \R^n_+.
\end{aligned}
\end{equation}
Since (\ref{formulation: SVM-soft margin}) is always feasible and the objective value is bounded from below, there exists an optimal solution $(\bm u^*, d^*, \bm \xi^*)$. We can show the following result.
\begin{corollary} \label{thm: equivalence SVM & QSSVM-vecl1 softmargin}
Let $(\bm u^*, d^*, \bm \xi^*)$ be an optimal solution of (\ref{formulation: SVM-soft margin}). It follows that $$(\bm 0_{n\times n},\bm u^*, d^*,\bm \xi^*)$$ solves (\ref{formulation:SQSSVM-vecL1}) when $\lambda$ is large enough.
\end{corollary}

Next, we consider the case when the data is quadratically separable with a sparse matrix $\Bf W$. In particular, we assume that the data set is separable by quadratic function $f(\bm x) = \frac{1}{2}\bm x^T {\Bf W} \bm x + \bm b^T \bm x + c$, where $\bm w = \hv({\Bf W})$ contains mostly $0$'s. Let $\mathcal Z$ be the set of indices of $0$'s in $\bm w$, i.e.,
$w_j = 0$ for all $j\in \mathcal Z$. In this case, the following restricted (\ref{formulation:QSSVM}) model is feasible and has an optimal solution with a finite objective value:
  \begin{equation} \tag{R-QSSVM\textprime\textprime}\label{formulation: R-QSSVM-vecl1-wbc}
       \begin{aligned}
          \min_{\bm w, \bm b, c} \quad &  q(\bm w, \bm b, c) \, \triangleq \,  \sum_{i=1}^m {\bm w}^T {\dataMi}^T \dataMi {\bm w} + 2 \sum_{i=1}^m {\bm w}^T {\dataMi}^T  {\bm b}  + m \bm b^T \bm b \\
          s.t. \quad & \datayi\lp \bm w^T \datasi + \bm b^T \dataxi + c \rp \, \geqslant \, 1, \quad i = 1, \dots, m, \\
          &  \bm w_j \, = \, 0\, , \quad \forall j\in \mathcal Z,\\
           & \bm w \in \R^{\frac{n(n+1)}{2}}, \, \bm b \in \R^n, \, c \in \R.
       \end{aligned}
\end{equation}
Let $(\bm w^*, \bm b^*, c^*)$ be an optimal solution of (\ref{formulation: R-QSSVM-vecl1-wbc}), there exists multipliers $\bm \alpha^* \in \R^m, $ and $\bm \beta_{\mathcal Z}^* \in \R^{|\mathcal Z|}$, such that the KKT conditions are satisfied. Expanding $\bm \beta_{\mathcal Z}^*$ to $\bm \beta^* \in \R^{\nct}$ by filling 0's at the indices not in $\mathcal Z$, the KKT conditions can be written below:
 \begin{equation*}
   \left\{
     \begin{aligned}
      &2 \sum_{i=1}^m {\dataMi}^T \dataMi \bm w^* + 2 \sum_{i=1}^m {\dataMi}^T \bm b^*  - \sum_{i=1}^m  \bm \alpha_i^* \datayi \datasi + \bm \beta^* = \bm 0,\\
      & 2\sum_{i=1}^m \dataMi \bm w^* + 2m \bm b^* = \sum_{i=1}^{m} \bm \alpha_i^* \datayi \dataxi,\\
      & \sum_{i=1}^{m} \bm \alpha_i^* \datayi = 0, \\
      & \bm \alpha^*_i \lp 1-\datayi \lp {\bm w^*}^T \datasi + {\bm b^*}^T \dataxi + c^* \rp \rp = 0, \quad  i = 1, \dots, m, \\
      & \bm \alpha^*\geqslant \bm 0, \\
      & 1-\datayi \lp {\bm w^*}^T \datasi + {\bm b^*}^T \dataxi + c^* \rp \leqslant 0,
      \quad  i = 1, \dots, m, \\
      &  \bm w_j = 0, \quad \forall j \in \mathcal Z.
     \end{aligned}
   \right.
   \end{equation*}
   Similar to Theorem \ref{thm: equivalence SVM & QSSVM-vecl1 hardmargin}, we can see that when $\lambda > \|\bm \beta^*\|_\infty$,
   $(\bm w^*, \bm b^*, c^*, \bm \alpha^*)$ satisfies the KKT conditions (\ref{L_Q KKT}), and hence $(\bm w^*, \bm b^*, c^*)$ is an optimal solution of (\ref{formulation:QSSVM-vecl1}). This indicates, that when $\lambda$ is large enough, the $\ell_1$ norm regularization can accurately capture the sparsity in matrix $\Bf W$.

   The above argument can also be applied for the model \ref{formulation:SQSSVM-vecL1} over any given data set. However, we can have a stronger result in quadratically separable case. In fact, by combining this result with the ones  obtained  from Theorems \ref{theorem:z_uniqueness}, \ref{th:surely_pd_G} and \ref{theorem:xi_0_mu_large}, when the separating quadratic surface is generated by a sparse matrix $\Bf W$, we have the below corollary.

   \begin{corollary}
   For almost any quadratically separable (\ref{eq:quad_separable}) data set $\mathcal D$  (\ref{def:dataset D}) for  which the generating matrix $\Bf W^*$ is sparse, the proposed model \ref{formulation:SQSSVM-vecL1} obtains a unique solution that captures this sparse matrix $\Bf W^*$ and also $\bm \xi^*=\bm 0$ provided that penalty parameters $\lambda$ and $\mu$ are large enough.
   \end{corollary}


\section{Numerical Experiments on Performance of \ref{formulation:QSSVM-vecl1} and \ref{formulation:SQSSVM-vecL1}} \label{sec:Numerical Experiments}

In this section, we conduct various numerical experiments  to analyze the behavior of \ref{formulation:QSSVM-vecl1} and \ref{formulation:SQSSVM-vecL1} over different data sets and demonstrate their effectiveness. All the experiments are conducted on a desktop computer employing all eight threads of Intel\textregistered\  Core\texttrademark\  i7-2600 CPU @ 3.40GHz and 8GB RAM. We use Gurobi 7.0.2 to solve the quadratic programs in all the QSSVM models.

Figures (\ref{fig:linear separable case 1})-(\ref{fig:quad separable case}) depict  the flexibility of \ref{formulation:QSSVM-vecl1} in capturing linear and quadratic separating surfaces. When data set is linearly separable (\ref{fig:linear separable case 1} and \ref{fig:linear separable case 2}), \ref{formulation:QSSVM-vecl1} yields a hyperplane for  $\lambda=10000$ but when it is quadratically separable (\ref{fig:quad separable case}), \ref{formulation:QSSVM-vecl1} with $\lambda=1$ behaves exactly the same as \ref{formulation:QSSVM}. This figure also confirms that SVM does not perform well when the data set is quadratically separable unlike its success in linearly separable data sets.

\begin{figure}[H]
    \centering
    \begin{subfigure}[t]{0.33\textwidth}
        \includegraphics[width=\textwidth]{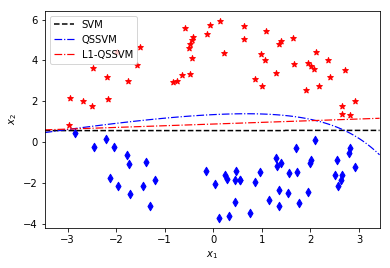}
        \caption{Linearly separable case 1. }
        \label{fig:linear separable case 1}
    \end{subfigure}
    ~
    \begin{subfigure}[t]{0.33\textwidth}
        \includegraphics[width=\textwidth]{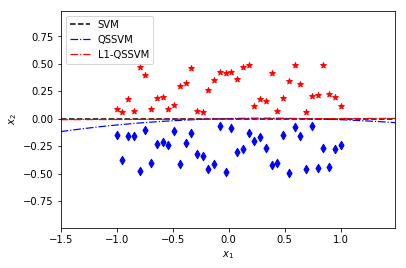}
        \caption{Linearly separable case 2. }
        \label{fig:linear separable case 2}
    \end{subfigure}
    ~
    \begin{subfigure}[t]{0.33\textwidth}
        \includegraphics[width=\textwidth]{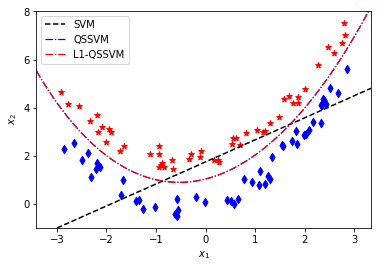}
        \caption{Quadratically separable case.}
        \label{fig:quad separable case}
    \end{subfigure}

    \caption{\ref{formulation:QSSVM-vecl1}  performance on linearly and quadratically separable data sets.}
    \label{fig:linear-quad compare}

\end{figure}

Figure \ref{fig:lambda Varies}  verifies Corollary \ref{thm: equivalence SVM & QSSVM-vecl1 softmargin} on  \ref{formulation:SQSSVM-vecL1}  in  visual details. Given a linearly separable data set and fix $\mu$, we plot the separating surfaces obtained from \ref{formulation:SQSSVM-vecL1} for  different values of $\lambda$ along with  the separating surfaces obtained from  \ref{formulation: SVM-hard margin} and \ref{formulation:SQSSVM}.  We can see that when $\lambda$ is small, the solution of L1-SQSSVM is close to that of \ref{formulation:SQSSVM} and when $\lambda$ is large, the solution of \ref{formulation:SQSSVM-vecL1} is close to that of SVM. In other words, as $\lambda$ gets bigger, the solution of \ref{formulation:SQSSVM-vecL1}  becomes flatter. Roughly speaking, the curvature approaches zero.

    \begin{table}[H]
	\centering
	\resizebox{\textwidth}{!}{
	\begin{tabular}{cccccc}
	\hline
	Data set & Artificial I & Artificial II & Artificial III & Artificial IV & Artificial 3-D    \\ \hline
	$n$ & 3 & 3 & 5 & 10 &  3 \\ \hline
	Sample size ($N_1$/$N_2$) & 67/58 & 79/71 & 106/81 & 204/171 & 99/101 \\ \hline
	\end{tabular}
	}
	\caption{Basic information of artificial data sets.}
	\label{table: 2-class artificial data sets info}
	\end{table}

We use Figure \ref{fig:sumxii goes to 0} to numerically verify Theorem \ref{theorem:xi_0_mu_large}. The data set utilized in this experiment is the Artificial 3-D data, which is quadratically separable. The basic information on the Artificial 3-D data utilized in this experiment is listed in Table \ref{table: 2-class artificial data sets info}. As shown in both pictures, the optimal solution of \ref{formulation:SQSSVM-vecL1} approaches  to that of \ref{formulation:QSSVM-vecl1}
as $\mu$ becomes large enough.

\begin{figure}[H]
  \centering
    \includegraphics[width=0.8\textwidth]{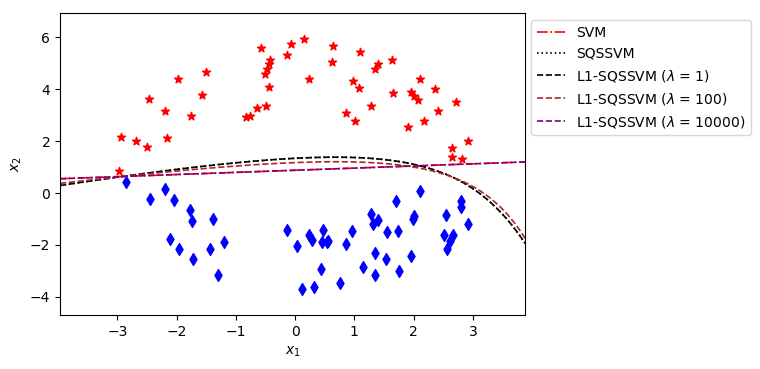}
  \caption{Influence of the parameter $\lambda$ on the curvature of the optimal solution of \ref{formulation:SQSSVM-vecL1}.}
  \label{fig:lambda Varies}
\end{figure}

\begin{figure}[H]
    \centering
   \begin{subfigure}[t]{0.40\textwidth}
    \includegraphics[width=\textwidth]{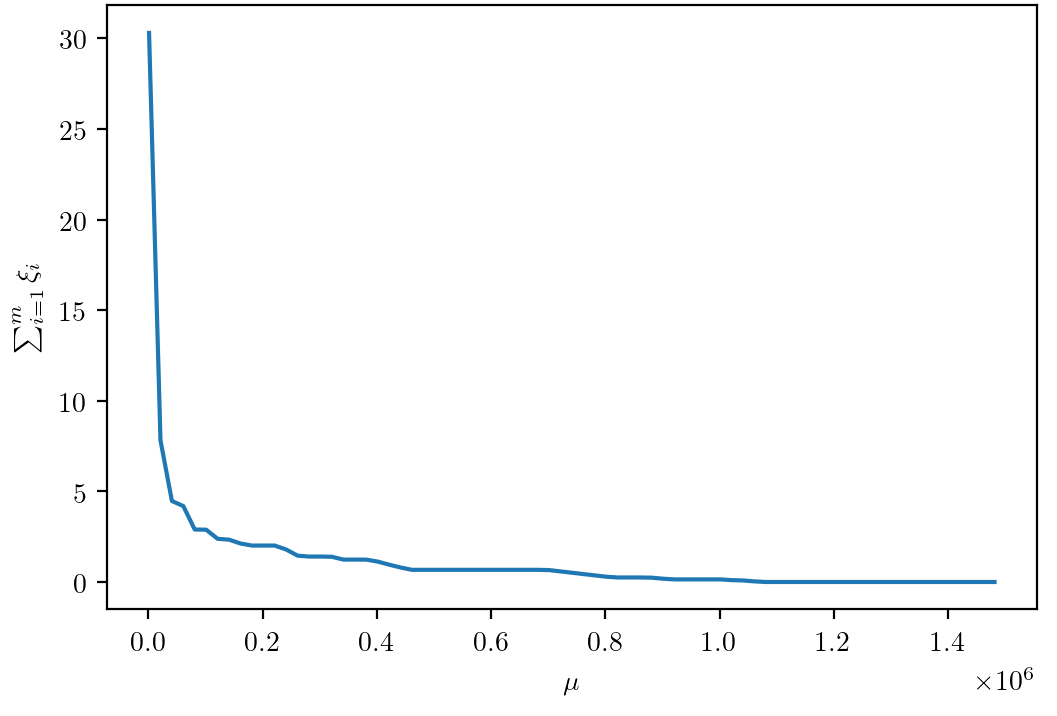}
    \caption{${\sum_{i=1}^m }\bm \xi_i$ against parameter $\mu$. (Artificial 3-D data)}
    \label{fig:sumxii goes to 0}
  \end{subfigure}
    ~
    \begin{subfigure}[t]{0.40\textwidth}
        \includegraphics[width=\textwidth]{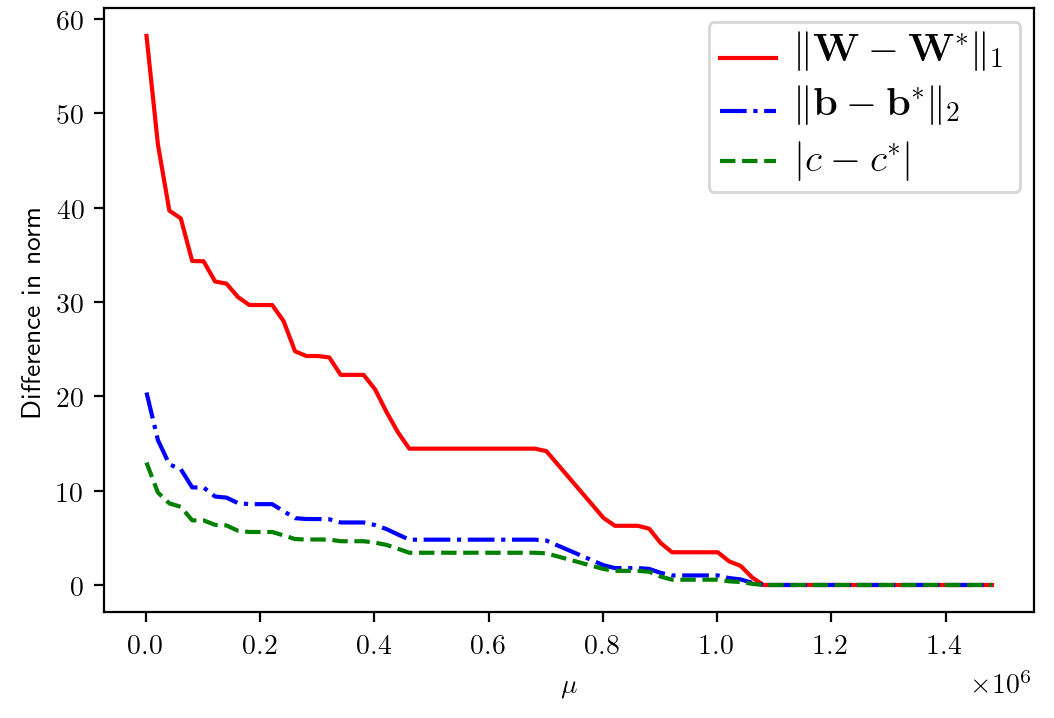}
        \caption{Relative error  against parameter $\mu$. (Artificial 3-D data)}
        \label{fig:norm diff goes to 0}
    \end{subfigure}

    \caption{Influence of the parameter $\mu$ on the behavior of the optimal solution of \ref{formulation:SQSSVM-vecL1}.}
    \label{fig:sumxii goes to 0 and coeff norm diff goes to 0}

\end{figure}

Next, we numerically demonstrate that the $\ell_1$ norm term in our proposed models is significant in classification, namely, a suitable parameter $\lambda>0$  exists  that  leads to  a better performance than when $\lambda=0$. We only focus on the soft margin  model because both models resemble similar behavior in this sense. To show that the existence of such optimum parameter $\lambda$ is independent of the choice of parameter $\mu$ and the data set, we have six different data sets in which the parameter $\mu$ changes from small to large. To tune a suitable approximation of optimum parameter $\mu$, we use \ref{formulation:SQSSVM} for a given data set and choose the $\hat \mu$ with the highest accuracy score. Note that our model  has two parameters (one more degree of freedom than that of \ref{formulation:SQSSVM}) so that it naturally improves the accuracy of classification compared with \ref{formulation:SQSSVM} with $\hat\mu$.
The considered discrete range for $\mu$ to obtain $\hat \mu$ is such that $\log_2\mu\in \{-3. -2, \dots, 20\}$. If distinct values of $\hat \mu$ exist, we simply set up $\hat \mu$ as their mean. Figure \ref{fig:accuracy score vs lambda} demonstrates that for different scales of $\hat \mu$, our proposed model \ref{formulation:SQSSVM-vecL1} for some $\lambda>0$ leads to a better performance than that of its parent model \ref{formulation:SQSSVM} in which $\lambda=0$ on artificial and real-world data sets.

\begin{figure}[H]
    \centering
    \begin{subfigure}[b]{0.3\textwidth}
        \includegraphics[width=\textwidth]{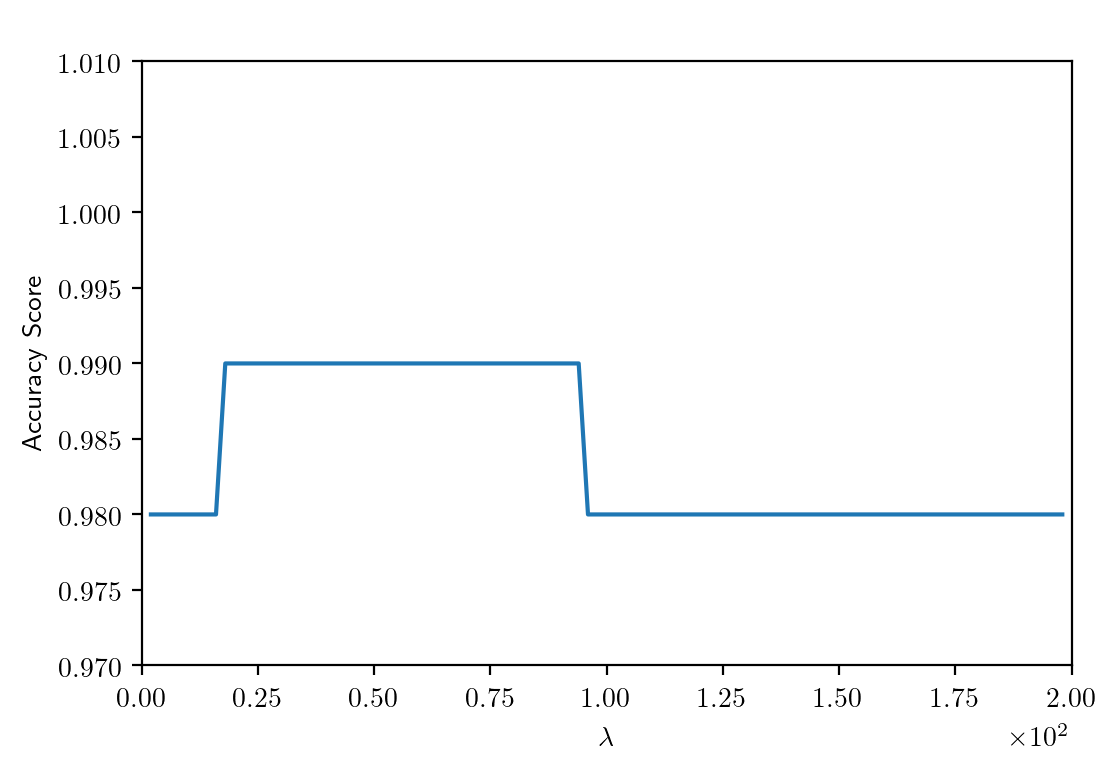}
        \caption{$\hat\mu = 128$\\ Iris data}
        \label{fig:iris-lambda-vs-score}
    \end{subfigure}
    ~
    \begin{subfigure}[b]{0.3\textwidth}
        \includegraphics[width=\textwidth]{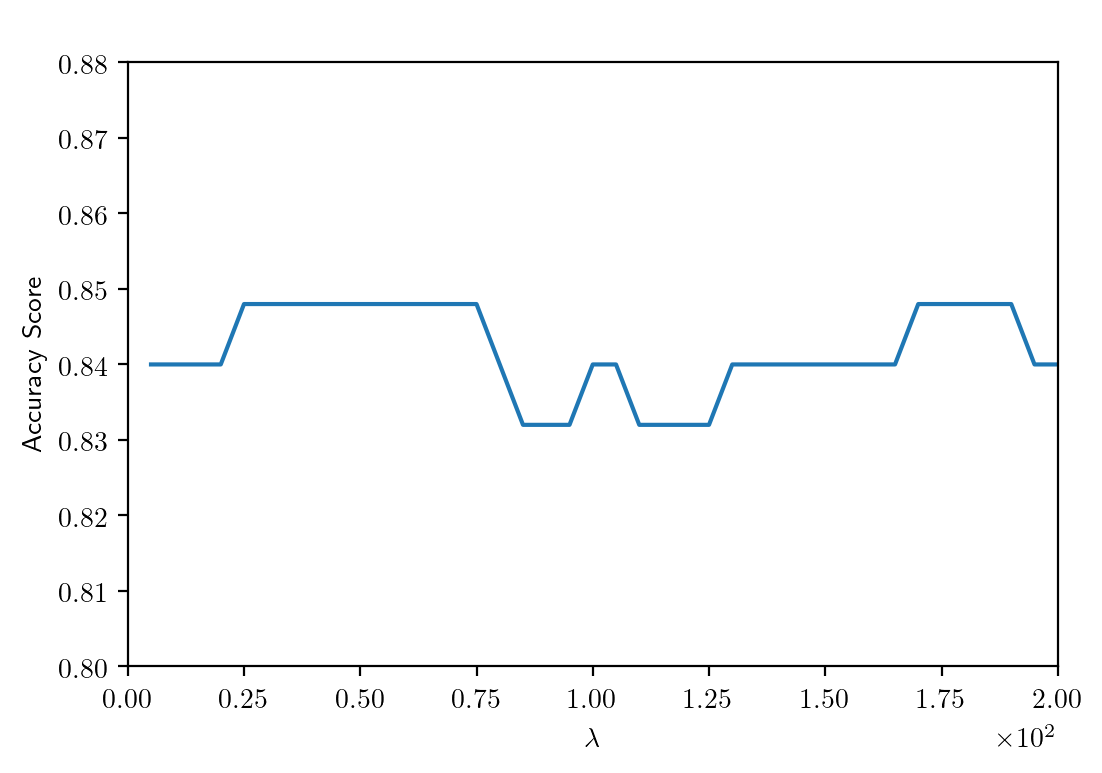}
        \caption{$\hat\mu = 203$\\ Artificial data \RNum{1}  }
        \label{fig:arti-SM1-lambda-vs-score}
    \end{subfigure}
    ~
    \begin{subfigure}[b]{0.3\textwidth}
        \includegraphics[width=\textwidth]{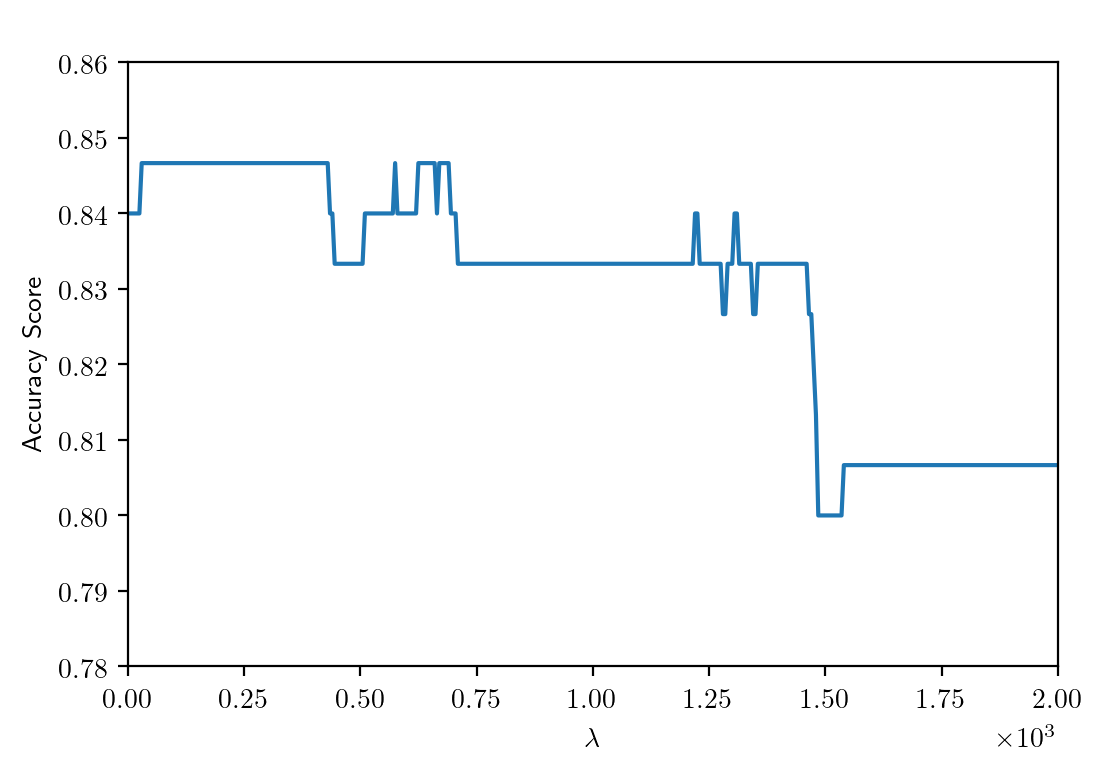}
        \caption{$\hat\mu = 1536$\\ Artificial data \RNum{2}  }
        \label{fig:arti-SM2-lambda-vs-score}
    \end{subfigure}

    \begin{subfigure}[b]{0.3\textwidth}
        \includegraphics[width=\textwidth]{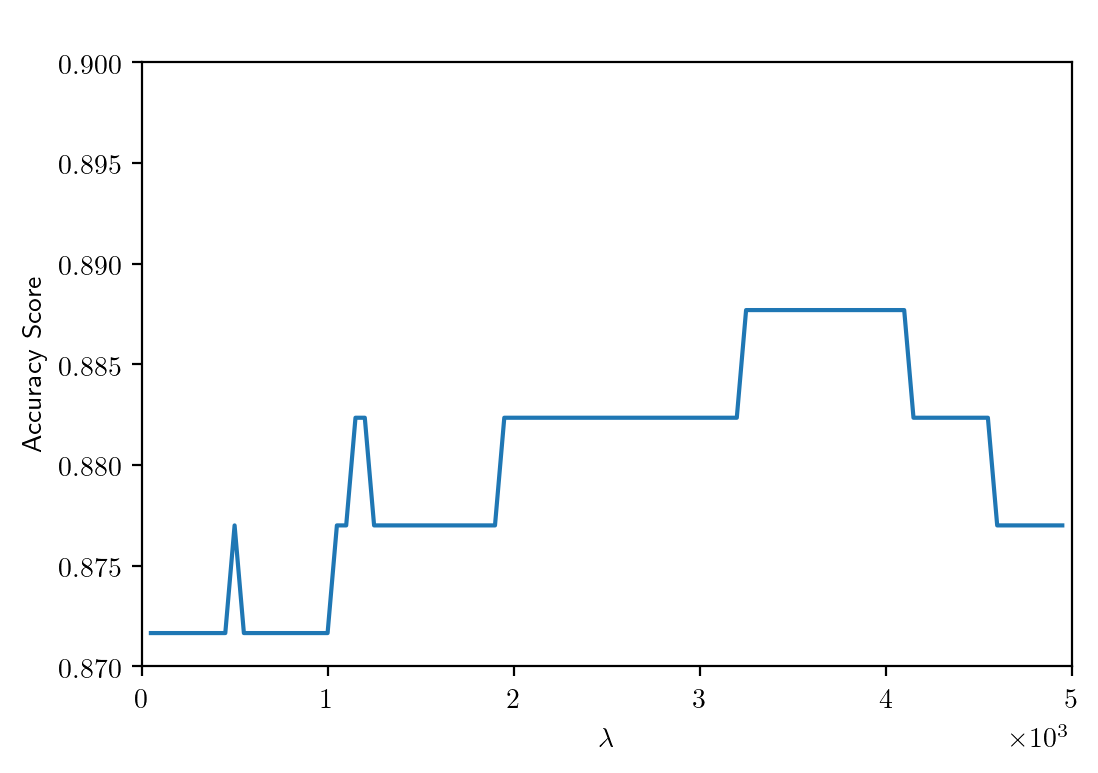}
        \caption{$\hat\mu = 9216$\\ Artificial data \RNum{3}  }
        \label{fig:arti-SM3-lambda-vs-score}
    \end{subfigure}
    ~
    \begin{subfigure}[b]{0.3\textwidth}
        \includegraphics[width=\textwidth]{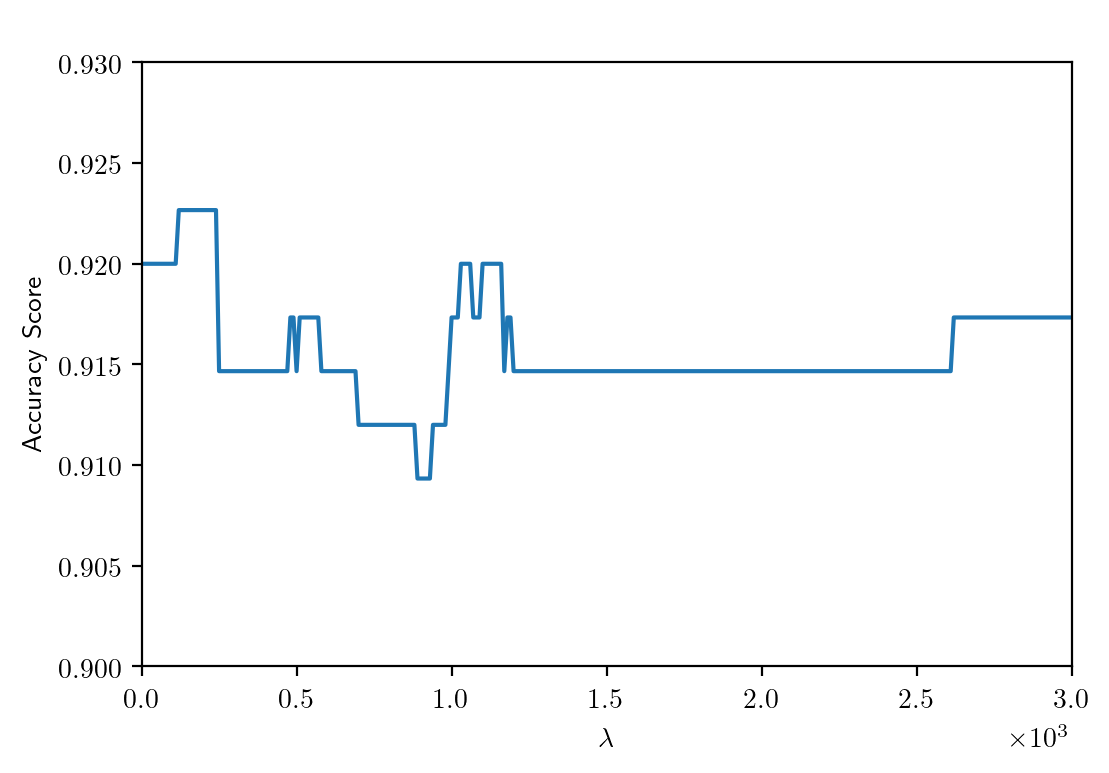}
        \caption{$\hat\mu = 32768$\\ Artificial data \RNum{4}  }
        \label{fig:arti-SM4-lambda-vs-score}
    \end{subfigure}
    ~
    \begin{subfigure}[b]{0.3\textwidth}
        \includegraphics[width=\textwidth]{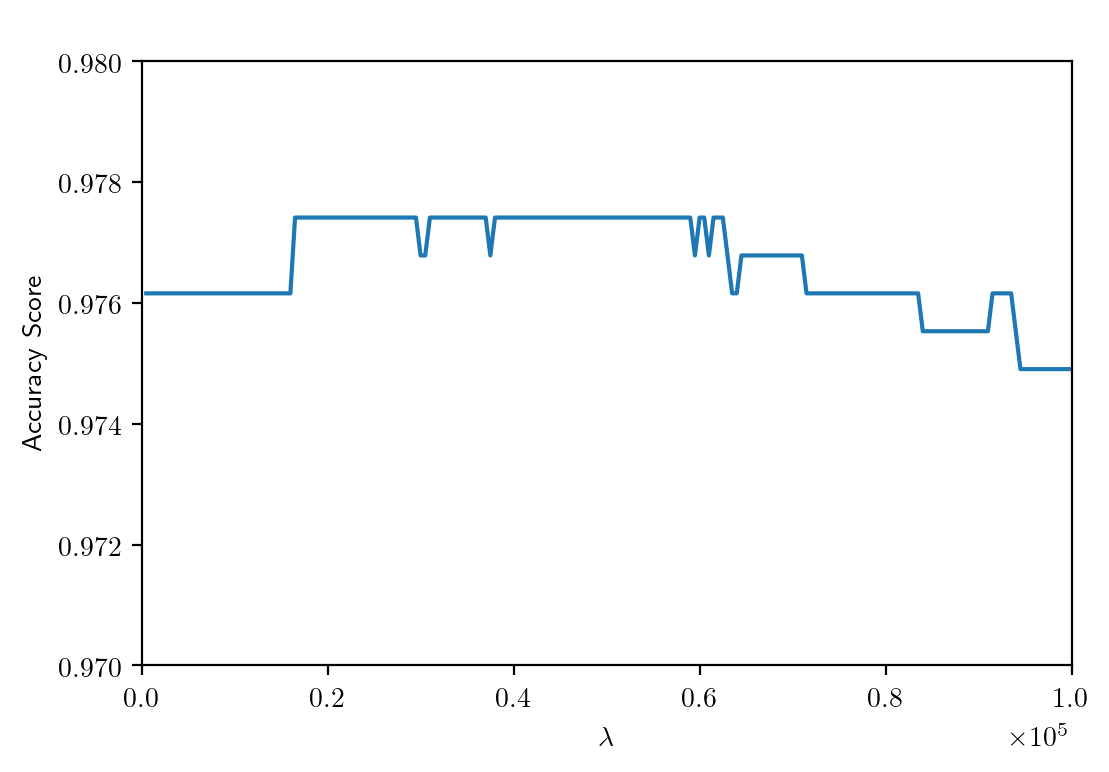}
        \caption{$\hat\mu = 131072$ \\ Car evaluation data}
        \label{fig:careval-lambda-vs-score}
    \end{subfigure}

    \caption{Accuracy score against the parameter $\lambda$}
    \label{fig:accuracy score vs lambda}
\end{figure}

Consider the case where a given data set is quadratically separable with a sparse $\Bf W$ matrix, i.e., when the separation surface $f(\bm x) = \frac{1}{2}\bm x^T \Bf W \bm x + \bm b^T \bm x + c = 0$ has a sparse $\Bf W$. By applying \ref{formulation:SQSSVM-vecL1} one can see that the $\ell_1$ norm regularization term enforces detecting  the true sparsity pattern of this matrix. To demonstrate  this property experimentally, we first generate a quadratic surface using $10$ features with the following sparse matrix $\Bf W$, vector $\bm b$ and  constant $c$:
\begin{equation} \label{example_sparsity_detection}
    \Bf W = { \begin{bmatrix}
      0 & 0 & 1 & 0 & 0 & 0 & 0 & 0 & 0 & 0\\
      0 & 0 & 0 & 2 & 0 & 0 & 0 & 0 & 0 & 0 \\
      1 & 0 & 0 & 0 & 3 & 0 & 0 & 0 & 0 & 0 \\
      0 & 2 & 0 & 0 & 0 & 4 & 0 & 0 & 0 & 0 \\
      0 & 0 & 3 & 0 & 0 & 0 & 5 & 0 & 0 & 0 \\
      0 & 0 & 0 & 4 & 0 & 0 & 0 & 6 & 0 & 0 \\
      0 & 0 & 0 & 0 & 5 & 0 & 0 & 0 & 7 & 0 \\
      0 & 0 & 0 & 0 & 0 & 6 & 0 & 0 & 0 & 0 \\
      0 & 0 & 0 & 0 & 0 & 0 & 7 & 0 & 0 & 0 \\
      0 & 0 & 0 & 0 & 0 & 0 & 0 & 0 & 0 & 0
    \end{bmatrix}}, \ \bm b = \begin{bmatrix}
         1 \\
         \vdots \\
         1 \\
         -1
       \end{bmatrix}, \ c = 2.
       \end{equation}

Next, we randomly generate $200$ data points on each side of the resulting quadratic separating  surface and another $100$ noisy data points  around this surface. then, using the same idea as explained before, we utilize \ref{formulation:SQSSVM} on this data set to tune parameter $\mu$ and obtain the best $\hat \mu$ (in terms of accuracy score). We finally apply \ref{formulation:SQSSVM-vecL1} with this $\hat \mu$ and obtain their  optimal matrices as $\lambda$ varies.  Figure \ref{fig:sparsity plot against lambda} shows that the sparsity of $\Bf W$ in (\ref{example_sparsity_detection}) is captured  as parameter $\lambda$ becomes larger.

\begin{figure}[H]
    \centering
    \begin{subfigure}[b]{0.22\textwidth}
        \includegraphics[width=\textwidth]{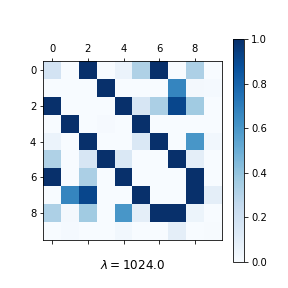}
        \caption{$\lambda = 1024$}
        \label{fig:sp1}
    \end{subfigure}
    \begin{subfigure}[b]{0.22\textwidth}
        \includegraphics[width=\textwidth]{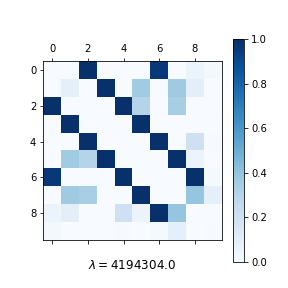}
        \caption{$\lambda = 4194304$}
        \label{fig:sp2}
    \end{subfigure}
    \begin{subfigure}[b]{0.22\textwidth}
        \includegraphics[width=\textwidth]{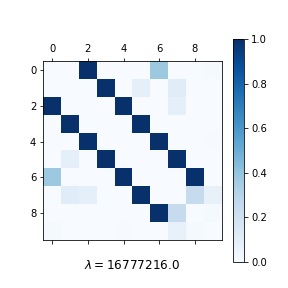}
        \caption{$\lambda = 16777216$}
        \label{fig:sp3}
    \end{subfigure}
    \begin{subfigure}[b]{0.22\textwidth}
        \includegraphics[width=\textwidth]{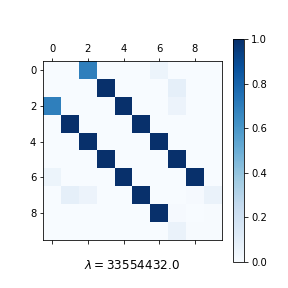}
        \caption{$\lambda = 33554432$}
        \label{fig:sp4}
    \end{subfigure}
    \caption{Sparsity pattern detection using \ref{formulation:SQSSVM-vecL1} as  parameter $\lambda$ varies.}
    \label{fig:sparsity plot against lambda}
\end{figure}

Lastly, the proposed \ref{formulation:SQSSVM-vecL1} is tested on five public benchmark data sets along with some well-known SVM models: \ref{formulation:SQSSVM}, \ref{formulation: SVM-hard margin}, and SVM with a Quadratic kernel (SVM-Quad). All the benchmark data sets are obtained from the UCI repository \cite{asuncion2007uci} and the basic information is listed in Table \ref{table: 2-class data sets info}. Notice that, for SVM-Quad we used SVC with 2-degree polynomial kernel Python package {\it{Scikit-learn}} \cite{scikit-learn}. We randomly pick $k\%$ ($k = 10, 20, 40$) \cite{dagher2008quadratic, luo2016soft} of the full data set as the training set. The parameters $\lambda$ and $\mu$ are tuned by using grid method \cite{luo2016soft, gao2020kernel}. The kernel parameters are selected by the package over training the full data set. In order to be statistically meaningful, for each fixed training rate $k$ on each model, the experiments are repeated for $50$ times. The mean, standard deviation, minimum and maximum of accuracy scores, and the average CPU time among these $50$ experiments are recorded. The accuracy score is defined as the rate of achieved correct labels by the model over the full data set. Note that the CPU time recorded in this paper does not include the time for tuning the parameters.


\begin{table}[H]
\centering
\begin{tabular}{cccc}
\hline
Data set                       & \# of features    & name of class & sample size \\ \hline
\multirow{2}{*}{Iris} & \multirow{2}{*}{4}
& versicolour  &  50           \\  &
& virginica    &  50           \\ \hline
\multirow{2}{*}{Car evaluation} & \multirow{2}{*}{6}
&      unacc         &  1210           \\
        &                    &    acc          &  384         \\ \hline
    \multirow{2}{*}{Diabetes} & \multirow{2}{*}{8}
&      yes        &  268          \\
        &                    &    no         &  500        \\ \hline
    \multirow{2}{*}{German Credit Data} & \multirow{2}{*}{20}
&      creditworthy        &  700          \\
        &                    &    non-creditworthy          &  300        \\ \hline
    \multirow{2}{*}{Ionosphere} & \multirow{2}{*}{34}
    &      good         &  225          \\
            &                    &    bad          &  126         \\ \hline
\end{tabular}
\caption{Description of 2-class data sets used.}
\label{table: 2-class data sets info}
\end{table}









\begin{table}[H]
\centering
\begin{tabular}{ccccccc}
\hline
\multirow{2}{*}{\begin{tabular}[c]{@{}c@{}}Training Rate\\  $k$\%\end{tabular}} & \multirow{2}{*}{Model} & \multicolumn{4}{c}{Accuracy score (\%)}   & \multirow{2}{*}{CPU time (s)} \\ \cline{3-6}

                         &   & mean  & std  & min   & max   &                               \\ \hline
\multirow{4}{*}{10}
 & L1-SQSSVM	 & 	91.93 & 5.49 & 63.33 & 	98.89 & 	0.058 \\
 & 	SQSSVM & 	89.33 & 4.07 & 	81.11 & 	96.67 & 	0.050 \\
 & 	SVM-Quad & 	89.49 & 4.91 & 	80.00 & 97.78 & 	0.003 \\
 & 	SVM & 	89.62 & 4.10 & 	78.89 &	97.78 & 	0.001 \\

\hline
\multirow{4}{*}{20}
& L1-SQSSVM	 & 	94.33 &	2.20 & 	90.00 & 	98.75 & 	0.063 \\
 & 	SQSSVM & 	92.60 & 2.57 & 	82.50 & 	96.25 & 	0.055 \\
 & 	SVM-Quad & 	93.03 & 2.72 & 	86.25 & 	98.75 & 	0.002 \\
 & 	SVM & 	93.00 & 3.01 & 	82.50 & 97.50 & 	0.002 \\

\hline
\multirow{4}{*}{40}
 &L1-SQSSVM	 & 	95.40 &	2.76 & 	86.67 & 	100.00 & 	0.075 \\
 & 	SQSSVM & 	93.97 & 3.73 & 	78.33 & 	100.00 & 	0.062 \\
 & 	SVM-Quad & 	94.30 & 3.38 & 	81.67 & 	98.33 & 	0.002 \\
 & 	SVM & 	94.50 & 3.29 & 	85.00 & 	100.00 & 	0.002 \\

\hline

\hline
\end{tabular}
\caption{Iris results.}
\label{table: iris data results}
\end{table}

\begin{table}[H]
\centering
\begin{tabular}{ccccccc}
\hline
\multirow{2}{*}{\begin{tabular}[c]{@{}c@{}}Training Rate\\  $k$\%\end{tabular}} & \multirow{2}{*}{Model} & \multicolumn{4}{c}{Accuracy score (\%)}   & \multirow{2}{*}{CPU time (s)} \\ \cline{3-6}

                         &   & mean  & std  & min   & max   &                               \\ \hline
\multirow{4}{*}{10}
 & L1-SQSSVM	 & 	90.48 & 2.13 & 	83.48 & 95.05 & 	0.961 \\
 & 	SQSSVM & 	90.48 & 2.35 & 	80.98 & 94.49 & 	0.937 \\
 & 	SVM-Quad & 	88.32 & 2.70 & 	80.98 & 93.45 & 	0.023 \\
 & 	SVM      & 	84.40 & 1.09 & 81.88 & 	86.90 & 0.001 \\

\hline
\multirow{4}{*}{20}
 & L1-SQSSVM	 & 	92.81 & 1.17 & 	89.50 & 95.30 & 	1.109 \\
 & 	SQSSVM & 	92.77 & 1.21 & 	89.58 & 95.30 & 	1.117 \\
 & 	SVM-Quad & 	92.30 & 1.14 & 	88.56 & 94.83 & 	0.001 \\
 & 	SVM      & 	85.08 & 0.91 & 83.23 & 	86.91 & 0.008 \\

\hline
\multirow{4}{*}{40}
 &L1-SQSSVM	 & 	95.80 &	0.73 & 	93.83 & 	97.07 & 	1.501 \\
 & 	SQSSVM & 	95.76 & 0.77 & 	93.83 & 	97.28 & 	1.521 \\
 & 	SVM-Quad & 	93.69 & 0.83 & 	91.43 & 	95.72 & 	0.087 \\
 & 	SVM & 	85.26 & 1.09 & 	81.71 & 	87.36 & 	0.003 \\

\hline

\hline
\end{tabular}
\caption{Car evaluation results.}
\label{table: careval results}
\end{table}

    \begin{table}[H]
    \centering
    \begin{tabular}{ccccccc}
    \hline
    \multirow{2}{*}{\begin{tabular}[c]{@{}c@{}}Training Rate\\  $k$\%\end{tabular}} & \multirow{2}{*}{Model} & \multicolumn{4}{c}{Accuracy score (\%)}   & \multirow{2}{*}{CPU time (s)} \\ \cline{3-6}

                             &   & mean  & std  & min   & max   &                               \\ \hline
    \multirow{4}{*}{10}
&	L1-SQSSVM	&	74.21	&	1.53	&	71.24	&	76.01	&	0.692	\\
&	SQSSVM	&	64.38	&	3.65	&	57.80	&	71.68	&	0.679	\\
&	SVM-Quad	&	66.07	&	4.53	&	57.66	&	71.53	&	0.102	\\
&	SVM	&	72.95	&	3.49	&	65.61	&	76.16	&	0.003	\\

    \hline
    \multirow{4}{*}{20}
&	L1-SQSSVM	&	76.28	&	0.63	&	75.12	&	77.07	&	0.924	\\
&	SQSSVM	&	69.40	&	2.49	&	65.85	&	72.52	&	0.950	\\
&	SVM-Quad	&	70.28	&	2.30	&	65.85	&	73.82	&	9.080	\\
&	SVM	&	74.86	&	1.68	&	71.54	&	77.07	&	0.009	\\

    \hline
    \multirow{4}{*}{40}
&	L1-SQSSVM	&	76.62	&	1.83	&	73.97	&	79.61	&	1.459	\\
&	SQSSVM	&	74.34	&	1.99	&	71.15	&	77.01	&	1.490	\\
&	SVM-Quad	&	75.21	&	1.23	&	73.54	&	77.22	&	86.561	\\
&	SVM	&	76.29	&	2.15	&	73.10	&	80.26	&	0.006	\\

    \hline

    \hline
    \end{tabular}
    \caption{Diabetes results.}
    \label{table: Pima Diabetes results}
    \end{table}

    \begin{table}[H]
    \centering
    \begin{tabular}{ccccccc}
    \hline
    \multirow{2}{*}{\begin{tabular}[c]{@{}c@{}}Training Rate\\  $k$\%\end{tabular}} & \multirow{2}{*}{Model} & \multicolumn{4}{c}{Accuracy score (\%)}   & \multirow{2}{*}{CPU time (s)} \\ \cline{3-6}

                             &   & mean  & std  & min   & max   &                               \\ \hline
    \multirow{4}{*}{10}
&	L1-SQSSVM	&	71.86	&	1.85	&	68.44	&	75.00	&	1.596	\\
&	SQSSVM	&	67.00	&	3.02	&	63.67	&	71.67	&	1.598	\\
&	SVM-Quad	&	68.29	&	2.61	&	64.00	&	72.44	&	0.006	\\
&	SVM	&	69.49	&	3.58	&	61.89	&	74.33	&	0.002	\\

    \hline
    \multirow{4}{*}{20}
&	L1-SQSSVM	&	73.88	&	1.29	&	71.38	&	75.88	&	2.572	\\
&	SQSSVM	&	67.55	&	2.78	&	62.88	&	72.88	&	2.541	\\
&	SVM-Quad	&	67.78	&	2.75	&	64.13	&	72.13	&	0.005	\\
&	SVM	&	73.86	&	1.22	&	71.25	&	75.88	&	0.005	\\

    \hline
    \multirow{4}{*}{40}
    &	L1-SQSSVM	&	74.86	&	1.25	&	72.00	&	77.00	&	4.622	\\
&	SQSSVM	&	65.99	&	2.66	&	61.17	&	69.83	&	4.456	\\
&	SVM-Quad	&	65.13	&	1.19	&	63.50	&	67.00	&	0.262	\\
&	SVM	&	74.73	&	1.07	&	73.50	&	77.00	&	0.005	\\

    \hline

    \hline
    \end{tabular}
    \caption{German Credit Data results.}
    \label{table: GCD results}
    \end{table}

    \begin{table}[H]
    \centering
    \begin{tabular}{ccccccc}
    \hline
    \multirow{2}{*}{\begin{tabular}[c]{@{}c@{}}Training Rate\\  $k$\%\end{tabular}} & \multirow{2}{*}{Model} & \multicolumn{4}{c}{Accuracy score (\%)}   & \multirow{2}{*}{CPU time (s)} \\ \cline{3-6}

                             &   & mean  & std  & min   & max   &                               \\ \hline
    \multirow{4}{*}{10}
     &	L1-SQSSVM	&	82.75	&	3.69	&	76.27	&	88.29	&	4.141	\\
    &	SQSSVM	&	79.24	&	3.15	&	74.37	&	83.86	&	3.945	\\
    &	SVM-Quad	&	83.48	&	2.39	&	78.48	&	78.48	&	0.003	\\
    &	SVM	&	80.09	&	2.24	&	75.95	&	82.28	&	0.006	\\

    \hline
    \multirow{4}{*}{20}
    &	L1-SQSSVM	&	87.90	&	3.72	&	80.07	&	92.53	&	5.096	\\
    &	SQSSVM	&	87.19	&	4.32	&	77.94	&	91.81	&	4.854	\\
    &	SVM-Quad	&	86.16	&	1.24	&	84.34	&	84.34	&	0.005	\\
    &	SVM	&	82.03	&	5.40	&	67.97	&	86.83	&	0.002	\\

    \hline
    \multirow{4}{*}{40}
     &L1-SQSSVM	 & 	90.28 &	3.33 & 	83.41 & 94.31 & 	7.063 \\
     & 	SQSSVM & 89.53 & 4.23 & 81.99 & 94.31 & 6.781 \\
     & 	SVM-Quad & 	86.40 & 3.03 & 	81.04 & 	91.00 & 	0.007 \\
     & 	SVM & 	83.60 & 3.46 & 	76.78 & 88.63 & 	0.006 \\

    \hline

    \hline
    \end{tabular}
    \caption{Ionosphere results.}
    \label{table: Ionosphere data results}
    \end{table}


From Tables \ref{table: iris data results} - \ref{table: Ionosphere data results}, we have the following observations:

\begin{itemize}
    \item The mean accuracy scores obtained by \ref{formulation:SQSSVM-vecL1} are the same or better than those of other models over all the tested benchmark data sets.
    
    \item The proposed L1-SQSSVM model produces the highest mean accuracy scores on the German Credit Data and the Ionosphere, the two data sets with many features. This indicates that the proposed model has the potential for classifying data sets with large number of features. 
    
    \item The training CPU times of the proposed \ref{formulation:SQSSVM-vecL1} on the tested data sets are acceptable.
\end{itemize}


\section{Conclusion}
This paper generalizes the standard kernel-free models of linear and quadratic surface support vector machines. The SVMs are only designed for (almost) linearly separable data sets and the QSSVMs only work for (almost) quadratically separable data sets cannot reduce to the corresponding hard or soft margin SVM if the data set is linearly separable. In other words, when the actual $\Bf W= \Bf 0$, the QSSVMs often output a surface with $\Bf W^* \ne \Bf 0$, which is not ideal for such generalizations of SVMs.

By incorporating an $\ell_1$ norm regularization in the objective function, we propose L1-QSSVM models that not only resolve these shortcomings but also account for possible sparsity patterns in setting appropriate penalty parameters. We further establish other interesting theoretical results such as solution existence, uniqueness, and vanishing margin for the soft margin version if the penalty parameter is large enough. To conclude the paper, we summarize all the obtained theoretical results for different types of data sets in the table below.


\begin{table}[h]
\centering
\resizebox{\textwidth}{!}{
\begin{tabular}{|c|c|c|}
\hline
 \backslashbox{Data set}{Model}  &
  $\begin{array}{c}
 \textrm{L1-QSSVM}
\end{array}$
  &
  $\begin{array}{c}
 \textrm{L1-SQSSVM}
\end{array}$
  \\
\hline \hline
$\begin{array}{c}
 \textrm{Linearly}\\
 \textrm{Separable}
\end{array}$
& $\begin{array}{l}
 \bullet \ \textrm{Solution existence}\\
 \bullet \ \bm z^* \textrm{ is almost always unique} \\
\bullet \ \textrm{Equivalence with SVM}  \\
 \textrm{\ \ \ for large enough} \ \lambda
\end{array}$ &
$\begin{array}{l}
\bullet \ \text{Solution existence} \\
\bullet \ \bm z^* \textrm{ is almost always unique}\\
\bullet \ \textrm{Equivalence with SSVM}  \\
 \textrm{\ \ for large enough} \ \lambda \\
\bullet \ \textrm{Solution is almost always} \\
\textrm{\ \ \ unique with } \bm \xi^*=\bm 0\text{ \ for} \\
\textrm{ \ \  large enough} \ \mu  
\end{array}$
\\
\hline
\hline
$\begin{array}{c}
 \textrm{Quadratically} \\
 \textrm{Separable}
\end{array}$  & $\begin{array}{l}
\bullet \textrm{ Solution existence} \\
 \bullet \ \bm z^* \textrm{ is almost always unique} \\
\bullet \ \textrm{Capturing possible sparsity} \\
\textrm{\ \ \ of} \ \Bf W^* \ \textrm{for large enough} \ \lambda \\
\end{array}$ &
$\begin{array}{l}
\bullet \ \text{Solution existence} \\
\bullet \ \bm z^* \textrm{ is almost always unique}\\
\bullet \ \textrm{Solution is almost always } \\
\textrm{\ \ \ unique with } \bm \xi^*=\bm 0\textrm{ for} \\
\textrm{ \ \ large enough} \ \mu \\
\bullet \ \textrm{Capturing possible sparsity} \\
\textrm{\ \ \ of} \ \Bf W^* \ \textrm{for large enough} \ \lambda \\
\end{array}$
\\
\hline \hline
$\begin{array}{c}
 \textrm{Neither}
\end{array}$
& $\begin{array}{l}
\\
\end{array}$ &
$\begin{array}{l}
 \bullet \textrm{ Solution existence} \\
 \bullet \ \bm z^* \textrm{ is almost always unique}
\end{array}$ \\
\hline
\end{tabular}
}
\caption{Summary of obtained theoretical  results in this paper.}
\label{table: obtained_results_summary}
\end{table}


 Therefore, along with the promising practical efficiency of these models as demonstrated in Section \ref{sec:Numerical Experiments}, we  conclude that the proposed L1-QSSVMs are justifiable in theory and effective in practice.

Our investigation of the proposed L1-SQSSVM model for binary classification leads to some potential research extensions. An immediate future work is to investigate the robustness of the proposed model on noisy data sets. Another interesting research direction is to apply the proposed model to some real-world applications, such as disease diagnosis, customer segmentation and more. Moreover, we notice that even though the computational efficiency is acceptable, it still needs to be improved. As a convex optimization problem, we plan to design a fast greedy algorithm \cite{gao2019rescaled} to solve it.

\section*{Acknowledgments} The authors would like to greatly thank Professor Jinglai Shen for reminding them of a result on the boundedness of the Lagrangian multipliers in convex programs under the Slater's condition. 

\newpage

\appendix\label{Appendix}
\section{Proof of Theorem \ref{th:pd_G}}\label{proof:thm:pd_G}
\begin{proof}
By the definition of $\dataMi$, for all $i=1,\cdots, m$, we have:
$$ \left(\dataHi\right)^T \dataHi \, = \,
\left[
\begin{array}{c}
\left(\dataMi\right)^T \\ \hdashline[2pt/2pt]
\mathbf I_n
\end{array}
\right]\left[
\begin{array}{c;{2pt/2pt}c}
\dataMi & \mathbf I_n
\end{array}
\right] \, = \, \left[
\begin{array}{c;{2pt/2pt}c}
  \left( \dataMi \right)^T \dataMi & \left(\dataMi\right)^T \\\hdashline[2pt/2pt]
  \dataMi & \mathbf I_n
\end{array}
\right].
$$
Therefore,
\begin{equation}\label{eq:G_definition}
\Bf G \, = \, 2 \sum_{i=1}^{m} \left(\dataHi\right)^T \dataHi \, = \, 2 \left[
\begin{array}{c;{2pt/2pt}c}
  \sum_{i=1}^{m} \left( \dataMi \right)^T \dataMi&  \sum_{i=1}^{m}\left(\dataMi\right)^T \\\hdashline[2pt/2pt]
  \sum_{i=1}^{m} \dataMi & m \mathbf I_n
\end{array}
\right].
\end{equation}
By Lemma \ref{lm:pd_partitioned}, it is clear that $\Bf G \succ 0$ if and only if
$$ \sum_{i=1}^{m} \left( \dataMi \right)^T \dataMi - \frac{1}{m}\left( \sum_{i=1}^{m}\dataMi\right)^T \left(\sum_{i=1}^{m} \dataMi\right) \, \succ \, 0.$$
By definitions, we have
\begin{equation}\nonumber
\sum_{i=1}^{m} \left( \dataMi \right)^T \dataMi \, = \, \sum_{i=1}^{m} \left( \dataXi \Bf D_n\right)^T \left(\dataXi \Bf D_n\right) \, = \, \Bf D_n^T \left[\sum_{i=1}^{m}\left( \dataXi \right)^T \dataXi \right] \Bf  D_n.
\end{equation}
Similarly, we have
\begin{eqnarray}
\left( \sum_{i=1}^{m}\dataMi\right)^T \sum_{i=1}^{m} \dataMi & = & \left( \sum_{i=1}^{m}\dataXi\Bf  D_n\right)^T \sum_{i=1}^{m} \dataXi\Bf  D_n \nonumber\\
 & = & \Bf D_n^T \left[\left(\sum_{i=1}^{m} \dataXi\right)^T \left(\sum_{i=1}^{m} \dataXi \right)\right] \Bf D_n \nonumber.
\end{eqnarray}
Since $\Bf D_n$ has full column rank, therefore
\begin{eqnarray*}
\Bf G \, \succ \, 0 & \Leftrightarrow & \Bf D_n^T \left[ \sum_{i=1}^{m}\left( \dataXi \right)^T \dataXi \, - \frac{1}{m} \left(\sum_{i=1}^{m} \dataXi\right)^T \left(\sum_{i=1}^{m} \dataXi \right)\right] \Bf D_n \, \succ \, 0 \nonumber\\
& \Leftrightarrow & \sum_{i=1}^{m}\left( \dataXi \right)^T \dataXi \, - \, \frac{1}{m} \left(\sum_{i=1}^{m} \dataXi\right)^T \left(\sum_{i=1}^{m} \dataXi \right) \, \succ \, 0.
\end{eqnarray*}
For any given vector $\bm v \in \R^{n^2}$, we partition $\bm v$ into $n$ parts with equal length, i.e.,
$$\bm v = [(\bm v^1)^T, (\bm v^2)^T, \cdots, (\bm v^n)^T]^T.$$
We have, for each $i=1,\cdots,n$
$$\dataXi \bm v \, = \, \left[
\begin{array}{ccc}
(\dataxi)^T \\
&\ddots & \\
&& (\dataxi)^T
\end{array}
\right]\left[
\begin{array}{c}
\bm v^1 \\
\vdots\\
\bm v^n
\end{array}
\right]\, = \, \left[
\begin{array}{ccc}
(\dataxi)^T \bm v^1\\
\vdots\\
(\dataxi)^T \bm v^n
\end{array}
\right].$$
Let $b_{ij} = \left(\dataxi\right)^T \bm v^j$ and
$\bm b^i = [b_{i1},\cdots, b_{in}]^T \in \R^n.$
We can see that $\dataXi \bm  v= \bm b^i$ and hence
$$\bm v^T \left[\sum_{i=1}^{m}\left( \dataXi \right)^T \dataXi \right] \bm v \, = \, \sum_{i=1}^{m} \bm v^T\left( \dataXi \right)^T \dataXi \bm v \, = \, \sum_{i=1}^m \|\bm b^i\|_2^2.$$
Also, we have
\begin{eqnarray}\label{eq:key_inequality}
\bm v^T \left[\left(\sum_{i=1}^{m} \dataXi\right)^T \left(\sum_{i=1}^{m} \dataXi \right)\right] \bm v & = & \left(\sum_{i=1}^{m} \dataXi \bm v\right)^T\left(\sum_{i=1}^{m} \dataXi \bm v\right) \nonumber \\
& = & \left(\sum_{i=1}^m \bm b^i\right)^T\left(\sum_{i=1}^m \bm b^i\right) \, = \, \left\|\sum_{i=1}^m \bm b^i\right\|_2^2\nonumber\\
& = & \sum_{i=1}^m \|\bm b^i\|^2_2 + \sum_{i\neq k} (\bm b^i)^T \bm  b^k \, \leqslant \, \sum_{i=1}^m \|\bm b^i\|^2_2 + \sum_{i\neq k} \|\bm b^i\|_2 \|\bm b^k\|_2\nonumber\\
& \leqslant & \sum_{i=1}^m \|\bm b^i\|^2_2  + \sum_{i\neq k} \frac{1}{2}\left(\|\bm b^i\|^2_2 + \|\bm b^k\|^2_2\right) = \, m\sum_{i=1}^m \|\bm b^i\|_2^2,
\end{eqnarray}
where  the equality holds if and only if $ \bm b^1 \, = \, \bm b^2 \, = \,\cdots \, = \,\bm  b^m.$
This is equivalent to
$$ \left(\bm x^{(1)}\right)^T\bm  v^j \, = \, \left(\bm x^{(2)}\right)^T \bm v^j \, = \, \cdots \, = \, \left(\bm x^{(m)}\right)^T \bm v^j, \quad j=1,\cdots, n.$$
Since $\Bf X$ has full column rank, this in turn implies that there exists $\bm u \in \R^n$ such that
$\Bf X \bm  u = \mathbf 1_m.$
However, by assumption $\mathbf 1_m$ is not in the column space of $\Bf X$, therefore the inequality in (\ref{eq:key_inequality}) holds strictly. That is,
$$\bm v^T \left[ \sum_{i=1}^{m}\left( \dataXi \right)^T \dataXi \, - \, \frac{1}{m} \left(\sum_{i=1}^{m} \dataXi\right)^T \left(\sum_{i=1}^{m} \dataXi \right) \right]\bm v>0,$$
for all $\bm v \in \R^{n^2}$ and $\bm v\neq 0$, i.e.,
$$\sum_{i=1}^{m}\left( \dataXi \right)^T \dataXi \, - \, \frac{1}{m} \left(\sum_{i=1}^{m} \dataXi\right)^T \left(\sum_{i=1}^{m} \dataXi \right) \, \succ\, 0.$$
This concludes the proof.

\end{proof}

\section{Proof of Theorem \ref{th:surely_pd_G}} \label{proof:thm:surely_pd_G}
\begin{proof}
From Theorem \ref{th:pd_G}, it suffices to show that the set of data matrices $\Bf  X$ not satisfying assumptions (A1) or (A2) is of Lebesgue measure zero in $\R^{m\times n}$. We augment $\Bf  X$ by appending an all 1 column to it. That is, we consider
$$\overline{\Bf  X} =
\left[
\begin{array}{c;{2pt/2pt}c}
  \Bf X & \bm 1_{m}
\end{array}
\right].$$
It is clear that assumptions (A1) and (A2) holds if and only if $\overline{\Bf X}$ has full column rank. We let
$$\mathcal S \, \triangleq \, \{\Bf  X\in \mathbb R^{m\times n} \ | \ \overline{\Bf  X} \text{ has linearly independent columns} \}.$$
Let $\mathcal I$ be an arbitrary subset of $\{1,\cdots,m\}$ with $n+1$ elements. Define the following polynomial:
$$ \varphi_{\mathcal I} (\Bf  X) \, \triangleq \, \det(\overline{\Bf X}_{\mathcal I \bullet}) \left( \sum_{i \notin \mathcal I} \sum_{j=1}^{n+1} \left[(\overline{ \Bf X}_{ij})^2 + 1\right] \right),$$
where $\overline{\Bf X}_{\mathcal I \bullet}$ is the submatrix formed by taking the rows with the index in $\mathcal I$ of $\overline{\Bf  X}$. It is clear that $\varphi_{\mathcal I} (\Bf  X) = 0$ if and only if $\det(\overline {\Bf  X}_{\mathcal I \bullet})=0$. Let the zero set of $\varphi_{\mathcal I}(\Bf X)$ be denoted by $z(\varphi_{\mathcal I}(\Bf  X))$. It is clear that this zero set is of Lebesgue measure 0 in $\R^{m\times n}$. Let the collection of all subset of $\{1,\cdots, m\}$ with $n+1$ elements be denoted by $\Theta$. We notice that the complement of $S$ in $\R^{m\times n}$ is
$$S^c \, = \, \left\{ \Bf  X \in \R^{m\times n} \, | \, \forall \, \mathcal J \in \Theta \mbox{ it holds that }\det(\overline{\Bf  X}_{\mathcal I \bullet}) = 0\right\}.$$
Thus,
$$S^c \subseteq \bigcap_{\mathcal I \in \Theta} z(\varphi_{\mathcal I}(\Bf  X)).$$
Since the right-hand side is a finite intersection of measure zero sets, it is still a measure zero set. Therefore $S^c$ is a measure zero set in $\R^{m\times n}$. This concludes the proof.

\end{proof}


\bibliographystyle{plain}
\bibliography{L1-SQSSVM}

\medskip
Received xxxx 20xx; revised xxxx 20xx.
\medskip

\end{document}